\documentclass[journal,onecolumn,12pt]{IEEEtran}
\date{}

\usepackage{amsfonts}
\usepackage{amsmath}
\usepackage{amssymb}
\usepackage{amsthm}
\usepackage{array}
\usepackage{caption}
\usepackage{cite}
\usepackage{citesort}
\usepackage{color}
\usepackage{epsfig}
\usepackage{enumitem} 
\usepackage{flexisym}
\usepackage{graphicx}
\usepackage{hyphenat}
\usepackage{ifpdf}
\usepackage{latexsym}
\usepackage{mathtools}
\usepackage{multirow}
\usepackage{psfrag}
\usepackage{setspace}
\usepackage{subfigure}
\usepackage{tikz}
\usepackage{times}
\usepackage{tkz-berge}
\usepackage{xparse}

\newtheorem{theorem}{Theorem}
\newtheorem{definition}{Definition}

\newtheorem{lemma}{Lemma}
\newtheorem{corollary}{Corollary}

\newtheorem{remark}{Remark}
\newtheorem{example}{Example}

\title{Characterization of Deterministic and Probabilistic Sampling Patterns for Finite Completability of Low Tensor-Train Rank Tensor}

\author{Morteza Ashraphijuo and Xiaodong Wang\thanks{The authors are with the Department of Electrical Engineering, Columbia University, NY, email: \{ashraphijuo,wangx\}@ee.columbia.edu. }}
 
\begin{document}
\maketitle


\begin{abstract}

In this paper, we analyze the fundamental conditions for  low-rank tensor completion given the separation or tensor-train (TT) rank, i.e., ranks of unfoldings. We exploit the algebraic structure of the TT decomposition to obtain the  deterministic necessary and sufficient  conditions on the locations of the samples to ensure finite completability. Specifically, we propose an algebraic geometric analysis on the TT manifold that can incorporate the whole rank vector simultaneously in contrast to the existing approach based on the Grassmannian manifold that can only incorporate one rank component. Our proposed technique characterizes the algebraic independence of a set of polynomials defined based on the sampling pattern and the TT decomposition, which is instrumental to obtaining the deterministic  condition on the sampling pattern for finite completability. In addition, based on the proposed analysis, assuming that the entries of the tensor are sampled independently with probability $p$, we derive  a lower bound on the sampling probability $p$, or equivalently, the number of sampled entries that ensures finite completability with high probability. Moreover, we also provide the deterministic and probabilistic conditions for unique completability.

\


\newpage

\end{abstract}

\section{Introduction}

Most of the literature on low-rank data completion (either matrix or tensor) propose optimization-based algorithms to construct a completion that matches the given samples and rank. For example, for the two-way tensor, i.e., matrix, many algorithms have been proposed that are based on convex relaxation of rank \cite{candes,candes2,cai,ashraphijuo2016c,phase} or alternating minimization \cite{jain2013low,ge2016matrix}. Similarly, for higher dimensional data a number of tensor completion algorithms exist that are based on different convex relaxations of the tensor ranks \cite{gandy,tomioka,nuctensor,romera} or other heuristics \cite{low,low2,goldfarb,7347424,wang2016tensor,liulow2}. The low-rank tensor completion problem has various applications, including compressed sensing \cite{lim,sid,gandy}, visual data reconstruction  \cite{visual,liulow2}, seismic data processing \cite{kreimer,ely20135d,wang2016tensor}, RF fingerprinting \cite{liu2016tensor,7347424}, reconstruction of cellular data \cite{vaneetcnsm}, etc..

Existing works on optimization-based matrix or tensor completion usually make  a set of strong assumptions on the correlations of the values of either the sampled or non-sampled entries (such as coherence) in order to provide a tensor that approximately fits in the sampled tensor. In contrast, here we are interested in investigating fundamental conditions on the sampling pattern that guarantee the existence of  finite or unique number of completions. Such conditions are ``fundamental'' in the sense that they are  independent of either the optimization formulation or the optimization algorithm used to compute the completion.  The matrix version of this problem has been treated in \cite{charact} and \cite{ashraphijuo2} for single-view and multi-view  data, respectively. Also, the tensor version of this problem under the Tucker rank has been treated in \cite{ashraphijuo}. In this paper, we investigate this problem for tensors under the tensor-train (TT) rank.

There are a number of tensor decompositions available, including Tucker decomposition or higher-order singular value decomposition \cite{Tuck,SVD,Tuckermanifold}, polyadic decomposition \cite{ten,kruskal}, tubal rank decomposition \cite{kilmer2013third} and several other representations \cite{Eck,de2,papa}. TT decomposition (also known as tree-train decomposition) was proposed in the field of quantum physics about $20$ years ago \cite{beck2n,scholy}. Later it was used in the area of machine learning \cite{oseledets,ose009king,oselesor}. A comprehensive survey on TT decomposition and the manifold of tensors of fixed TT rank can be found in \cite{TT} that also includes a comparison between the TT and Tucker decompositions for a better understanding of the advantages of TT decomposition. 


In this paper, we propose a geometric analysis on the TT manifold to study the problem of low-rank tensor completion given the TT rank. We first briefly mention the differences and new challenges for the TT model in comparison with the Tucker model considered in \cite{ashraphijuo}. In Tucker decomposition which is the high-order singular value decomposition, we have a $d$-way tensor as the core which is the generalization of the basis for matrices. However, in TT decomposition we are dealing with two and three-way tensors and besides, the geometries of the Tucker and TT manifolds are totally different. Moreover, the notions of tensor multiplications in these two decompositions are different, and therefore the polynomials that can be obtained through each observed entry have completely different structures (to study algebraic independence).

Let $\mathcal{U}$ denote the sampled tensor and $\Omega$ denote the binary sampling pattern tensor that is of the same dimension and size as $\mathcal{U}$. The entries of $\Omega$ that correspond to the observed entries of $\mathcal{U}$ are equal to $1$ and the rest of the entries are set as $0$. This paper is mainly concerned with the following three problems.

{\bf Problem (i):} Given the TT rank, characterize the necessary and sufficient conditions on the sampling pattern $\Omega$, under which there exist only finitely many completions of $\mathcal{U}$.

We define a polynomial for each sampled entry such that the variables of the polynomial are the entries of the two or three-dimensional tensors in the TT decomposition. Then, we propose a geometric method on the TT manifold to obtain the maximum number algebraically independent polynomials (among all the defined polynomials for any of the sampled entries) in terms of the geometric structure of the sampling pattern $\Omega$. Finally, we show that if the maximum number algebraically independent polynomials meets a threshold, which depends on the structure of the sampling pattern $\Omega$, the sampled tensor $\mathcal{U}$ is finitely completable. We emphasize the fact that the proposed algebraic geometry analysis on the TT manifold is not a simple generalization of the existing analysis on the Grassmannian or Tucker manifold as almost every step needs to be developed anew.

{\bf Problem (ii):} Given the TT rank, characterize a sufficient conditions on the sampling pattern $\Omega$, under which there exists only one completion of $\mathcal{U}$.

We use the developed tools for solving Problem (i) and in addition to the condition for finite completability, we add more polynomials (samples) in a way such that the corresponding minimally algebraically dependent set of polynomials leads to that all involved variables can be determined uniquely.

{\bf Problem (iii):} Provide a lower bound on the total number of sampled entries such that the proposed conditions on the sampling pattern $\Omega$ for finite and unique completability are satisfied with high probability.

Assuming that the entries of $\mathcal{U}$ are sampled independently with probability $p$, we develop lower bounds on $p$ such that the deterministic conditions for Problems (i) and (ii) are met with high probability. 



The remainder of this paper is organized as follows. In Section \ref{1n}, the preliminaries and problem statement are presented. Problems (i), (ii) and (iii) are treated in Sections \ref{findetsec}, \ref{secprob} and \ref{secsuni}, respectively. Some numerical results are provided in Section \ref{secsimu}. Finally, Section \ref{seccon} concludes the paper.


\section{Background}\label{1n}
\subsection{Preliminaries and Notations}\label{notations}

In this paper, it is assumed that a $d$-way tensor $\mathcal{U} \in \mathbb{R}^{n_1  \times \cdots \times n_d}$ is sampled. For the sake of simplicity in notation, define  $N_{i} \triangleq \left( \Pi_{j=1}^{i} \  n_j \right)$ and $\bar N_{i} \triangleq \left( \Pi_{j=i+1}^{d} \  n_j \right)$. Also, for any real number $x$, define $x^+\triangleq\max\{0,x\}$.

Define the matrix $\mathbf{\widetilde U}_{(i)} \in \mathbb{R}^{N_i\times  \bar N_{i}}$ as the $i$-th unfolding of the tensor $\mathcal{U}$, such that $\mathcal{U}(\vec{x}) = \\ \mathbf{\widetilde U}_{(i)}({\widetilde  M}_{i} (x_1,\dots,x_i),{\widetilde{{M}} }_{-i} (x_{i+1},\ldots,x_d))$, where ${\widetilde M}_{i}: (x_1,\dots,x_i) \rightarrow  \{1,2,\dots, N_i\}$ and ${\widetilde{{M}}}_{-i}: (x_{i+1},\ldots,x_d) \\ \rightarrow  \{1,2,\dots, \bar N_{i} \}$ are two bijective mappings and $\mathcal{U}(\vec{x})$ represents an entry of tensor $\mathcal{U}$  with coordinate $\vec{x}=(x_1,\dots,x_d)$.

The separation or tensor-train (TT) rank of a tensor is defined as $\text{rank} (\mathcal{U})=(r_1,\ldots,r_{d-1})$ where $r_i = \text{rank}(\mathbf{\widetilde U}_{(i)})$, $i=1,\dots,d-1$. Note that $r_i \leq \max \{N_i, \bar N_{i} \}$ in general and also $r_1$ is simply the conventional matrix rank when $d=2$. The TT decomposition of a tensor $\mathcal{U}$ is given by
\begin{eqnarray}\label{TTeq1}
\mathcal{U} = \mathcal{U}^{(1)} \dots  \mathcal{U}^{(d)},
\end{eqnarray}
where the $3$-way tensors $\mathcal{U}^{(i)} \in \mathbb{R}^{r_{i-1} \times n_i \times r_{i}}$ for $i=2,\dots,d-1$ and matrices $\mathcal{U}^{(1)} \in \mathbb{R}^{n_1 \times r_1}$  and $\mathcal{U}^{(d)} \in \mathbb{R}^{r_{d-1} \times n_d}$ are the components of this decomposition and furthermore the tensor product in \eqref{TTeq1} is defined as
\begin{eqnarray}\label{TTeq2}
\mathcal{U}(\vec{x}) =  \sum_{k_1=1}^{r_1} \cdots \sum_{k_{d-1}=1}^{r_{d-1}}  \mathcal{U}^{(1)}(x_1,k_1)  \left( \prod_{i=2}^{d-1} \mathcal{U}^{(i)}(k_{i-1},x_i,k_i) \right) \mathcal{U}^{(d)}(k_{d-1},x_d).
\end{eqnarray}

Observe that the above tensor multiplication $\mathcal{C} = \mathcal{A} \mathcal{B}$ for two tensors $\mathcal{A} \in \mathbb{R}^{m_1 \times \dots \times m_i}$ and  $\mathcal{B} \in \mathbb{R}^{t_1 \times \dots \times t_j}$ is similar to the simple matrix multiplication across the $i$-th dimension of $\mathcal{A}$ and the first dimension of $\mathcal{B}$. Hence, $m_i=t_1$, $\mathbf{\widetilde A}_{(i-1)} \mathbf{\widetilde B}_{(1)} = \mathbf{\widetilde C}_{(i-1)}$ and also $\mathcal{C} \in \mathbb{R}^{m_1 \times \dots m_{i-1} \times t_2 \times \dots \times t_j}$. For notational simplicity, we denote $\mathbb{U}=(\mathcal{U}^{(1)}, \dots ,  \mathcal{U}^{(d)})$. Given the order $d$ and dimension sizes $n_1,\dots , n_d$, the space of all tensors of fixed TT rank vector $r = (r_1, \dots , r_{d-1})$ is a manifold of dimension \cite{TT}
\begin{eqnarray}\label{dimtt}
\sum_{i=1}^{d} r_{i-1}n_ir_i -\sum_{i=1}^{d-1} r_i^2,
\end{eqnarray}
where $r_0=r_d \triangleq 1$. As $\mathcal{U}^{(1)}$ and $\mathcal{U}^{(d)}$ are two-way tensors, we can also denote them by $\mathbf{U}^{(1)}$ and $\mathbf{U}^{(d)}$ in this paper.

Denote $\Omega$ as the binary sampling pattern tensor that is of the same size as $\mathcal{U}$ and $\Omega(\vec{x})=1$ if $\mathcal{U}(\vec{x})$ is observed and  $\Omega(\vec{x})=0$ otherwise. $\mathbf{X}(1:m,:)$ denotes the first $m$ rows of the matrix $\mathbf{X}$ and $\mathbf{X}^{\top}$ denotes the transpose of $\mathbf{X}$.

Let $\mathbf{U}_{(i)}$ be the $i$-th matricization of the tensor $\mathcal{U}$, i.e., the matrix $\mathbf{U}_{(i)}$ has $n_i$ rows and $\frac{N_d}{n_i}$ columns such that $\mathcal{U}(\vec{x}) = {\mathbf{U}}_{(i)}(x_i,{M}_{i} (x_1,\ldots,x_{i-1},x_{i+1},\ldots,x_d))$, where ${M}_{i}: (x_1,\ldots,x_{i-1},x_{i+1},\ldots,x_d) \rightarrow \{1,2,\dots, \frac{N_d}{n_i}\}$  is a bijective mapping. Observe that for any arbitrary tensor $\mathcal{A}$, the first matricization and the first unfolding are the same, i.e., $\mathbf{ A}_{(1)} = \mathbf{\widetilde A}_{(1)}$.   

\subsection{Problem Statement and A Motivating Example}\label{example}

We are interested in finding deterministic and probabilistic conditions on the sampling pattern tensor $\Omega$ under which there are finite completions of the sampled tensor $\mathcal{U}$ that satisfy $\text{rank}(\mathcal{U})=(r_1,r_2,\dots,r_{d-1})$. 

First we compare the following two approaches in an example to emphasize the necessity of our analysis for general order tensors: (i) analyzing each unfolding individually with the rank constraint of the corresponding unfolding, (ii) analyzing via TT decomposition that incorporates all rank components simultaneously. In particular, we will show via an example that analyzing each of the unfoldings separately is not enough to  guarantee finite completability when all rank components are given, while we show that for the same example TT decomposition ensures finite completability.

Consider a $3$-way tensor $\mathcal{U} \in \mathbb{R}^{2 \times 2 \times 2}$ with TT rank $(1,1)$. Assume that four entries of this tensor are observed: $(1,1,1),(2,1,1),(1,2,1),$ and $(1,1,2)$. Observe that the first unfolding $\mathbf{\widetilde U}_{(1)} \in \mathbb{R}^{2 \times 4}$  is also the first matricization $\mathbf{U}_{(1)}$. Moreover, the second unfolding $\mathbf{\widetilde U}_{(2)} \in \mathbb{R}^{4 \times 2}$ is the transpose of the third matricization $\mathbf{U}_{(3)}^{\top}$ since $\mathcal{U}$ is a three-way tensor. Therefore, the first and second components of the TT rank are the first and the third components of the Tucker rank, respectively.

It is shown in Section II of \cite{ashraphijuo} that having any $4$ entries of a rank-$1$ matrix, there are infinitely many completions for it. As a result, the first and second unfoldings each is infinitely many completable given only the corresponding rank constraint. Note that the analysis on Grassmannian manifold in \cite{charact} is not capable of incorporating more than one rank constraint. However, as we show next the intersection of the mentioned two infinite sets (having both of the rank constraints) is a finite set.

We take advantage of both elements of TT rank simultaneously, in order to show there exist only finitely many completions. Given the TT rank $(1,1)$, we define $\mathcal{U}^{(1)}=[x \  x^{\prime}]^{\top} \in \mathbb{R}^{2 \times 1}$, $\mathcal{U}^{(2)}(1,1,1)=y$, $\mathcal{U}^{(2)}(1,2,1)=y^{\prime}$ (where $\mathcal{U}^{(2)}\in \mathbb{R}^{1 \times 2 \times 1}$) and $\mathcal{U}^{(3)}=[z \  z^{\prime}] \in \mathbb{R}^{1 \times 2}$. Using the decomposition in  \eqref{TTeq1}, we have the followings
\begin{align}
\mathcal{U}(1,1,1) &= xyz, & \mathcal{U}(2,2,1) &= x^{\prime}y^{\prime}z, \\ \nonumber
\mathcal{U}(2,1,1) &= x^{\prime}yz,  & \mathcal{U}(2,1,2) &= x^{\prime}yz^{\prime} , \\ \nonumber
\mathcal{U}(1,2,1) &= xy^{\prime}z, & \mathcal{U}(1,2,2) &= xy^{\prime}z^{\prime}, \\ \nonumber
\mathcal{U}(1,1,2) &= xyz^{\prime},  & \mathcal{U}(2,2,2) &= x^{\prime}y^{\prime}z^{\prime}. \nonumber
\end{align}

Recall that $(1,1,1),(2,1,1),(1,2,1),$ and $(1,1,2)$ are the observed entries. Hence, the unknown entries can be determined uniquely in terms of the $4$ observed entries as
\begin{eqnarray}
\mathcal{U}(2,2,1) &=& x^{\prime}y^{\prime}z = \frac{\mathcal{U}(2,1,1)\mathcal{U}(1,2,1)}{\mathcal{U}(1,1,1)}, \\ \nonumber
\mathcal{U}(2,1,2) &=& x^{\prime}yz^{\prime} = \frac{\mathcal{U}(2,1,1)\mathcal{U}(1,1,2)}{\mathcal{U}(1,1,1)}, \\ \nonumber
\mathcal{U}(1,2,2) &=& xy^{\prime}z^{\prime} = \frac{\mathcal{U}(1,2,1)\mathcal{U}(1,1,2)}{\mathcal{U}(1,1,1)}, \\ \nonumber
\mathcal{U}(2,2,2) &=& x^{\prime}y^{\prime}z^{\prime} = \frac{\mathcal{U}(2,1,1)\mathcal{U}(1,2,1)\mathcal{U}(1,1,2)}{\mathcal{U}(1,1,1)\mathcal{U}(1,1,1)}. \nonumber
\end{eqnarray}


\section{Deterministic Conditions for Finite Completability}\label{findetsec}

This section characterizes the connection between the sampling pattern and the number of solutions of a low-rank tensor completion. In Section \ref{indepbernbef}, we define a polynomial based on each observed entry. Then, given the rank vector, we transform the problem of finite completability of $\mathcal{U}$ to the problem of including enough number of algebraically independent polynomials among the defined polynomials for the observed entries. In Section \ref{consttens}, we construct a constraint tensor based on the sampling pattern $\Omega$. This tensor is useful for analyzing the algebraic independency of a subset of polynomials among all defined polynomials. In Section \ref{algebraind}, we show the relationship between the number of algebraically independent polynomials in the mentioned set of polynomials and finite completability of the sampled tensor.

\subsection{Geometry of TT Manifold}\label{indepbernbef}
 
Here, we briefly mention some facts to highlight the fundamentals of our proposed analysis. Recall that $r_0=r_d=1$.
 
\begin{itemize}
\item {\bf Fact $1$}: As it can be seen from  \eqref{TTeq2}, any observed entry $\mathcal{U}(\vec{x})$ results in an equation that involves $r_{i-1}r_i$ entries of $\mathcal{U}^{(i)}$, $i=1,\dots,d$. Considering the entries of $\mathbb{U}$ as variables (right-hand side of  \eqref{TTeq2}), each observed entry results in a polynomial in terms of these variables.

\item {\bf Fact $2$}: As it can be seen from  \eqref{TTeq2}, for any observed entry $\mathcal{U}(\vec{x})$, the value of $x_i$ specifies the location of the $r_{i-1}r_i$ entries of $\mathcal{U}^{(i)}$ that are involved in the corresponding polynomial, $i=1,\dots,d$. In other words, the value of $x_i$ specifies the row number of the second (first) matricization of $\mathcal{U}^{(i)}$ which its $r_{i-1}r_i$ entries are involved in the corresponding polynomial, $i=2,\dots,d$ ($i=1$).


\item {\bf Fact $3$}: It can be concluded from Bernstein's theorem \cite{Bernstein} that in a system of $n$ polynomials in $n$ variables with coefficients chosen generically, the polynomials are algebraically independent with probability one, and therefore there exist only finitely many solutions. Moreover, in a system of $n$ polynomials in $n-1$ variables (or less), polynomials are algebraically dependent with probability one. 

\end{itemize}
 
Given all observed entries $\{\mathcal{U}(\vec{x}): \Omega(\vec{x}) = 1 \}$, we are interested in finding the number of possible solutions in terms of entries of $\mathbb{U}$ (infinite or finite) via investigating the algebraic independence among these polynomials.

We are interested in providing a structure on the decomposition $\mathbb{U}$ such that there is one decomposition among all possible decompositions of the sampled tensor $\mathcal{U}$ that captures the structure. Before describing such a structure on TT decomposition, we start with a similar structure for matrix decomposition.

\begin{lemma}\label{basel0}
Let $\mathbf{X}$ denote a generically chosen matrix from the manifold of $n_1 \times n_2$ matrices of rank $r$. Then, there exists a unique decomposition $\mathbf{X}= \mathbf{YZ}$ such that $\mathbf{Y} \in \mathbb{R}^{n_1 \times r}$, $\mathbf{Z} \in \mathbb{R}^{r \times n_2}$ and $\mathbf{Y}(1:r,1:r) = \mathbf{I}_r$, where $\mathbf{Y}(1:r,1:r)$ represents the submatrix of $\mathbf{Y}$ consists of the first $r$ columns and the first $r$ rows and $\mathbf{I}_r$ denotes the $r \times r$ identity matrix.
\end{lemma}

\begin{proof}
We show that there exists exactly one decomposition $\mathbf{X}= \mathbf{YZ}$ such that $\mathbf{Y}(1:r,1:r) = \mathbf{I}_r$ with probability one. Considering the first $r$ rows of $\mathbf{X}= \mathbf{YZ}$, we conclude $\mathbf{X}(1:r,:) = \mathbf{I}_r \mathbf{Z} = \mathbf{Z}$. Therefore, we need to show that there exists exactly one $\mathbf{Y}(r+1:n_1,:)$ such that $\mathbf{X}(r+1:n_1,:) = \mathbf{Y}(r+1:n_1,:) \mathbf{Z}$ or equivalently $\mathbf{X}(r+1:n_1,:)^{\top} = \mathbf{X}(1:r,:)^{\top} \mathbf{Y}(r+1:n_1,:)^{\top}$. It suffices to show that each column of $\mathbf{Y}(r+1:n_1,:)$ can be determined uniquely having $\mathbf{x} = \mathbf{X}(1:r,:)^{\top} \mathbf{y}$ where $\mathbf{x} \in \mathbb{R}^{n_2 \times 1}$ and $\mathbf{y} \in \mathbb{R}^{r \times 1}$. As $\mathbf{X}$ is a generically chosen $n_1 \times n_2$ matrix of rank $r$, we have $\text{rank}\left(\mathbf{X}(1:r,:)\right) = r$ with probability one. Hence, $\mathbf{x}(1:r) = \mathbf{X}(1:r,1:r)^{\top} \mathbf{y}$ results in $r$ independent degree-$1$ equations in terms of the $r$ variables (entries of $\mathbf{y}$), and therefore $\mathbf{y}$ has exactly one solution with probability one.
\end{proof}

\begin{remark}\label{counterexstr}
Note that the genericity assumption is necessary as we can find counter examples for Lemma \ref{basel0} in the absence of genericity assumption, e.g., it is easily verified that the following decomposition is not possible:
\begin{center}
\begin{tabular}{ |c|c|c|c| } 
 \hline
$1$ & $2$ & $1$ & $3$ \\ \hline
$1$ & $2$ & $1$ & $3$ \\ \hline
$1$ & $2$ & $2$ & $3$ \\ \hline
$1$ & $2$ & $2$ & $3$ \\ 
 \hline
\end{tabular} 
 \  \ \ $=$ \ \ \ 
\begin{tabular}{ |c|c| } 
 \hline
$1$ &  $0$ \\ \hline
$0$ & $1$ \\ \hline
$y_1$ & $y_2$ \\ \hline
$y_3$ & $y_4$  \\ 
 \hline
\end{tabular}
 \  \ \ $\times$ \ \ \ 
 \begin{tabular}{ |c|c|c|c| } 
 \hline
$x_1$ & $x_2$ & $x_3$ & $x_4$ \\ \hline
$x_5$ & $x_6$ & $x_7$ & $x_8$ \\ 
 \hline
\end{tabular}
\end{center}
\end{remark}

\begin{remark}\label{lemgenunistr}
Assume that $\mathbf{Q} \in \mathbb{R}^{r \times r}$ is an arbitrary given full rank matrix. Then, for any submatrix{\footnote{Specified by a subset of rows and a subset of columns (not necessarily consecutive).}} $\mathbf{P} \in \mathbb{R}^{r \times r}$ of $\mathbf{Y}$, Lemma \ref{basel0} also holds if we replace $\mathbf{Y}(1:r,1:r) = \mathbf{I}_r$ by $\mathbf{P}=\mathbf{Q}$ in the statement. The proof is similar to the proof of Lemma \ref{basel0} and thus it is omitted.
\end{remark}

As mentioned earlier, similar to the matrix case, we are interested in obtaining a structure on TT decomposition of a tensor such that there exists one decomposition among all possible TT decompositions of a tensor that captures the structure. Hence, we define the following structure on the decomposition in order to characterize a condition on the sampling pattern to study the algebraic independency of the above-mentioned polynomials.

\begin{definition}\label{defstrucpropTT}
Consider any $d-1$ submatrices $\mathbf{P}_{1},\dots,\mathbf{P}_{d-1}$ of $\mathbf{U}^{(1)},\mathbf{U}^{(2)}_{(2)},\mathbf{U}^{(3)}_{(2)},\dots,\mathbf{U}^{(d-1)}_{(2)}$, respectively such that (i) $\mathbf{P}_i \in \mathbb{R}^{r_i \times r_i}$, $i=1,\dots,d-1$, (ii)  the $r_i$ columns of $\mathbf{U}^{(i)}_{(2)}$ corresponding to columns of $\mathbf{P}_i$ belong to $r_i$ disjoint rows of $\mathbf{U}^{(i)}_{(3)}$, $i=2,\dots,d-1$. Then, $\mathbb{U}$ is said to have a proper structure if $\mathbf{P}_i$ is full rank, $i=1,\dots,d$. {\footnote{Since $\mathcal{U}^{(1)}$ and $\mathcal{U}^{(d)}$ are two-way tensors, i.e., matrices we also denote them by $\mathbf{U}^{(1)}$ and $\mathbf{U}^{(d)}$. Moreover, since $\mathcal{U}^{(i)}$ is a three-way tensor, $\mathbf{\widetilde U}^{(i)}_{(2)} = \mathbf{U}^{{(i)}^{\top}}_{(3)}$, $i=2,\dots,d-1.$}} 
\end{definition}

Define the following matrices:
\begin{eqnarray}\label{canonicalTT1}
\mathbf{P}_i^{\text{can}} (x_i,x_i^{\prime}) = \mathcal{U}^{(i)}(1,x_i,x_i^{\prime}) \in \mathbb{R}^{r_i \times r_i}, \ \ \ \ \ \ \ \ i = 2,\dots,d-1,
\end{eqnarray}
and
\begin{eqnarray}\label{canonicalTT2}
\mathbf{P}_1^{\text{can}} (x_1,x_1^{\prime}) = \mathcal{U}^{(1)}(x_1,x_1^{\prime}) \in \mathbb{R}^{r_1 \times r_1},
\end{eqnarray}
where $1 \leq x_i \leq r_i$ and $1 \leq x_i^{\prime} \leq r_i$. It is easy to verify that $\mathbf{P}_{1}^{\text{can}},\dots,\mathbf{P}_{d-1}^{\text{can}}$ satisfy properties (i) and (ii) in Definition \ref{defstrucpropTT}.

\begin{definition}\label{classp1}
{\bf (Canonical basis)} We call $\mathbb{U}$ a canonical decomposition if for $i =1,\dots,d$ we have $\mathbf{P}_i^{\text{can}}  = \mathbf{I}_{r_i}$, where $\mathbf{I}_{r_i}$ is the $r_i \times r_i$ identity matrix.
\end{definition}

\begin{lemma}\label{minimalpropTT}
Consider the TT decomposition in \eqref{TTeq2}. Then, $\mathbf{U}^{(1)} \in \mathbb{R}^{n_1 \times r_1}$, $\mathbf{U}^{(d)} \in \mathbb{R}^{r_{d-1} \times n_{d}}$, $\mathbf{U}^{(i)}_{(1)} \in \mathbb{R}^{ r_{i-1} \times n_i r_i}$ and $\mathbf{U}^{(i)}_{(3)} \in \mathbb{R}^{r_i \times r_{i-1} n_i}$, $i = 2,\dots,d-1$, are full rank matrices.  
\end{lemma}

\begin{proof}
In general, besides the separation rank $(r_1,\dots,r_{d-1})$, we may be able to obtain a TT decomposition for other vectors $(r_1^{\prime},\dots,r_{d-1}^{\prime})$ as well. However, according to \cite{TT} among all possible TT decomposition for different values of $r_i^{\prime}$'s, $r_i^{\prime} = \text{rank}(\mathbf{\widetilde U}_{(i)}) = r_i$, $i=1,\dots,d-1$, is minimal, in the sense that  there does not exist any decomposition with $r_i^{\prime}$'s such that $r_i^{\prime} \leq r_i$ for $i=1,\dots,d-1$ and $r_i^{\prime} < r_i$ for at least one $i \in \{1,\dots,d-1\}$. By contradiction, assume that $\mathbf{U}^{(i+1)}_{(1)}$ is not full rank. Then, $\text{rank}\left( \mathbf{ \widetilde U}^{(i)}_{(2)}  \mathbf{U}^{(i+1)}_{(1)} \right) < r_i$.

Let $\mathbf X$ denote the matrix $\mathbf{ \widetilde U}^{(i)}_{(2)}  \mathbf{U}^{(i+1)}_{(1)}$. Since $\text{rank}\left( \mathbf X \right) = r_i^{\prime} < r_i $, there exists a decomposition $ \mathbf{X} = \mathbf{ \widetilde U}^{{(i)}^{\prime}}_{(2)}  \mathbf{U}^{{(i+1)}^{\prime}}_{(1)}$ such that $\mathbf{ \widetilde U}^{{(i)}^{\prime}}_{(2)} \in \mathbb{R}^{r_{i-1}n_{i} \times r_i^{\prime}}$ and also $\mathbf{U}^{{(i+1)}^{\prime}}_{(1)} \in \mathbb{R}^{ r_i^{\prime} \times n_{i+1}r_{i+1}}$. Hence, the existence of the TT decomposition with $\mathcal{  U}^{(i)}$ and $  \mathcal{U}^{(i+1)}$ replaced by $\mathcal{  U}^{{(i)}^{\prime}} $ and $ \mathcal{U}^{{(i+1)}^{\prime}}$ contradicts the above-mentioned minimum property of the separation rank. Note that for a three-way tensor, the second unfolding is the transpose of the third matricization, and therefore $\text{rank}\left( \mathbf{ \widetilde U}^{(i)}_{(2)}   \right)= \text{rank}\left( \mathbf{  U}^{(i)}_{(3)}   \right)$ and the rest of the cases can be verified similarly.
\end{proof}


\begin{lemma}\label{pattern}
Assume that $\mathbf{Q}_i \in \mathbf{R}^{r_i \times r_i}$ is an arbitrary given full rank matrix, $1 \leq i \leq d-1$. Consider a set of matrices $\mathbf{P}_{1},\dots,\mathbf{P}_{d-1}$ that satisfy properties (i) and (ii) in Definition \ref{defstrucpropTT}. Then, there exists exactly one decomposition $\mathbb{U}$ of the sampled tensor $\mathcal{U}$ such that $\mathbf{P}_i=\mathbf{Q}_i$, $i=1,\dots,d-1$.
\end{lemma}

\begin{proof}
Consider an arbitrary decomposition $\mathbb{U}$ of the sampled tensor $\mathcal{U}$. Let $\mathcal{A}^{(i)} = \mathcal{U}^{(i)} \mathcal{U}^{(i+1)} \in \mathbb{R}^{r_{i-1} \times n_i \times n_{i+1} \times r_{i+1}}$, $i=1,\dots,d-1$, where the above multiplication is the same tensor multiplication in TT decomposition \eqref{TTeq1}. Note that for a three-way tensor, the second unfolding is the transpose of the third matricization, and therefore their ranks are the same. According to Lemma \ref{minimalpropTT}, $\text{rank}\left(\mathbf{U}^{(1)} \right)=\text{rank}\left(\mathbf{U}^{(2)}_{(1)}\right)=r_1$, $\text{rank}\left(\widetilde{\mathbf{U}}^{(2)}_{(2)} \right)=\text{rank}\left(\mathbf{U}^{(3)}_{(1)}\right)=r_2$, $\dots$, and $\text{rank}\left(\widetilde{\mathbf{U}}^{(d-1)}_{(2)} \right)=\text{rank}\left(\mathbf{U}^{(d)}\right)=r_{d-1}$. 

As a result, we have $\text{rank}\left(\mathbf{U}^{(1)} \mathbf{U}^{(2)}_{(1)}\right)=r_1$, $\text{rank}\left(\widetilde{\mathbf{U}}^{(2)}_{(2)} \mathbf{U}^{(3)}_{(1)}\right)=r_2, \dots,$ $\text{rank}\left(\widetilde{\mathbf{U}}^{(d-1)}_{(2)} \mathbf{U}^{(d)}\right)=r_{d-1}$. Observe that $\widetilde{\mathbf{U}}^{(i)}_{(2)} \mathbf{U}^{(i+1)}_{(1)} = \widetilde{\mathbf{A}}^{(i)}_{(2)}$, and therefore $\text{rank}\left( \widetilde{\mathbf{A}}^{(i)}_{(2)} \right) =  r_{i}$ for $i=2,\dots,d-2$ and similarly $\text{rank}\left( \widetilde{\mathbf{A}}^{(1)}_{(1)} \right) =  r_{1}$ and $\text{rank}\left( \widetilde{\mathbf{A}}^{(d)}_{(2)} \right) =  r_{d}$. According to Lemma \ref{basel0} and Remark \ref{lemgenunistr}, for an $n_1 \times n_2$ matrix $\mathbf{X}$ of rank $r$ there exists a unique decomposition $\mathbf{X}=\mathbf{X}_1 \mathbf{X}_2$ such that $\mathbf{X}_1 \in \mathbb{R}^{n_1 \times r}$ and $\mathbf{X}_2 \in \mathbb{R}^{r \times n_2}$ and an arbitrary $r \times r$ submatrix of $\mathbf{X}_1$ is equal to the given $r \times r$ full rank matrix.

We claim that there exist $(\mathcal{V}^{(i)},\mathcal{V}^{(i+1)})$ such that $\mathcal{V}^{(i)} \mathcal{V}^{(i+1)} = \mathcal{A}^{(i)}$ and the corresponding submatrix $\mathbf{P}_i$ is equal to the given full rank matrix $\mathbf{Q}_i$, $i=1,\dots,d-1$. We repeat this procedure for each $i=1,\dots,d-1$ and update two components of TT decomposition $(\mathcal{V}^{(i)},\mathcal{V}^{(i+1)})$ at iteration $i$ and at the end, we obtain a TT decomposition that has the mentioned structure in the statement of Lemma \ref{pattern}. In the following we show the existence of such $(\mathcal{V}^{(i)},\mathcal{V}^{(i+1)})$ at each iteration. At step one, we find $(\mathcal{V}^{(1)},\mathcal{V}^{(2)})$ such that $\mathcal{V}^{(1)} \mathcal{V}^{(2)} = \mathcal{A}^{(1)}$ and the corresponding submatrix $\mathbf{P}_1$ of $\mathcal{V}^{(1)}$ is equal to $\mathbf{Q}_1$. We update the decomposition with $\mathcal{U}^{(1)}$ and $\mathcal{U}^{(2)}$ replaced by $\mathcal{V}^{(1)}$ and $\mathcal{V}^{(2)}$, and therefore we obtain a new decomposition $\mathbb{U}^1$ of the sampled tensor $\mathcal{U}$ such that the submatrix of $\mathcal{V}^{(1)}$ corresponding to $\mathbf{P}_1$ is equal to $\mathbf{Q}_1$. Then, in step $2$ we consider $\mathcal{A}^{(2)}$ and similarly we update the second and third term of the decomposition obtained in the last step. Eventually after $d-1$ steps, we obtain a decomposition of the sampled tensor $\mathcal{U}$ that $\mathbf{P}_i = \mathbf{Q}_i$, $i=1,\dots,d-1$.To show the uniqueness of such decomposition, we show that each component of the TT decomposition can be determined uniquely. Remark \ref{lemgenunistr} for rank component $r_1$ results that $\mathcal{U}^{(1)}$ and the multiplication of the rest of the components of the TT decomposition can be determined uniquely. By repeating this procedure for other rank components the uniqueness of such decomposition can be verified by showing the uniqueness of the components one by one.
\end{proof}

Lemma \ref{pattern} leads to the fact that given $\mathcal{U}^{(d)}$, the dimension of all tuples $(\mathcal{U}^{(1)},\dots ,\mathcal{U}^{(d-1)})$ that satisfy TT decomposition is $  \sum_{i=1}^{d-1} r_{i-1}n_ir_i -\sum_{i=1}^{d-1} r_i^2 $, as $  \sum_{i=1}^{d-1} r_{i-1}n_ir_i$ is the total number of entries of $(\mathcal{U}^{(1)} , \dots , \mathcal{U}^{(d-1)})$ and $\sum_{i=1}^{d-1} r_i^2 $ is the total number of the entries of the pattern or structure that is equivalent to the uniqueness of TT decomposition. We make the following assumption which will be referred to, when it is needed. 

{\underline{\em Assumption $1$}}: {\it Each column of $\widetilde{\mathbf{U}}_{(d-1)}$ includes at least $r_{d-1}$ observed entries.}

\begin{remark}\label{remass}
Note that given $\mathcal{U}^{(1)},\dots ,\mathcal{U}^{(d-1)}$, polynomials in \eqref{TTeq2} are degree-$1$ polynomials in terms of the entries of $\mathcal{U}^{(d)}$. Hence, Assumption $1$ results in $n_dr_{d-1}$ degree-$1$ polynomials in terms of the entries of $\mathcal{U}^{(d)}$. As a result, the entries of $\mathcal{U}^{(d)}$ can be determined uniquely in terms of the entries of $(\mathcal{U}^{(1)} , \dots , \mathcal{U}^{(d-1)})$.
\end{remark}

\begin{definition}\label{setpolTT}
Let $\mathcal{P}(\Omega)$ denote the set of polynomials corresponding to the observed entries as in \eqref{TTeq2} excluding the $n_dr_{d-1}$ observed entries of Assumption $1$. Note that since $\mathcal{U}^{(d)}$ is already solved in terms of $(\mathcal{U}^{(1)} , \dots , \mathcal{U}^{(d-1)})$, each polynomial in $\mathcal{P}(\Omega)$ is in terms of elements of $(\mathcal{U}^{(1)} , \dots , \mathcal{U}^{(d-1)})$.
\end{definition} 

The following lemma provides the necessary and sufficient condition on $\mathcal{P}(\Omega)$ for finite completability of the sampled tensor $\mathcal{U}$.

\begin{lemma}\label{finminpolned}
Suppose that Assumption $1$ holds. For almost every $\mathcal{U}$, there exist only finitely many completions of $\mathcal{U}$ if and only if there exist $  \sum_{i=1}^{d-1} r_{i-1}n_ir_i -\sum_{i=1}^{d-1} r_i^2 $ algebraically independent polynomials in $\mathcal{P}(\Omega)$.
\end{lemma}

\begin{proof}
The proof is omitted due to the similarity to the proof of Lemma $2$ in \cite{ashraphijuo}. The only minor difference is that here the dimension is $  \sum_{i=1}^{d-1} r_{i-1}n_ir_i -\sum_{i=1}^{d-1} r_i^2 $ instead of $  \left(\Pi_{i=1}^{j} n_i\right) \left(  \Pi_{i=j+1}^{d} r_i\right)  - \left( \sum_{i=j+1}^{d}   r_i^{2} \right) $ which is the dimension of the core for Tucker decomposition.
\end{proof}

\subsection{Constraint Tensor}\label{consttens}

In the following, we provide a procedure to construct a binary tensor ${\breve{\Omega}}$ based on $\Omega$ such that $\mathcal{P}({\breve{\Omega}}) = \mathcal{P}(\Omega) $ and each polynomial can be represented by one $d$-way subtensor of ${\breve{\Omega}}$ that belongs to $\mathbb{R}^{n_1 \times \dots \times n_{d-1} \times 1}$. Using ${\breve{\Omega}}$, we are able to recognize the observed entries that have been used to obtain the $\mathcal{U}^{(d)}$ in terms of the entries of $\mathcal{U}^{(1)},\dots ,\mathcal{U}^{(d-1)}$, and we can easily verify if two polynomials in $\mathcal{P}(\Omega)$ are in terms of the same set of variables. Then, in Section \ref{algebraind}, we characterize the relationship between the maximum number of algebraically independent polynomials in $\mathcal{P}({\breve{\Omega}})$ and ${\breve{\Omega}}$.

For each subtensor $\mathcal{Y}$ of the sampled tensor $\mathcal{U}$, let $N_{\Omega}(\mathcal{Y})$ denote the number of sampled entries in $\mathcal{Y}$. Specifically, consider any subtensor $\mathcal{Y} \in \mathbb{R}^{n_1 \times n_2 \times \cdots \times n_{d-1} \times 1 }$ of the tensor $\mathcal{U}$. Then, $\mathcal{Y}$ contributes $N_{\Omega}(\mathcal{Y}) - r_{d-1}$ polynomial equations in terms of the entries of $\mathcal{U}^{(1)},\dots ,\mathcal{U}^{(d-1)}$ among all $N_{\Omega}(\mathcal{U}) -r_{d-1}n_d$ polynomials in $\mathcal{P}(\Omega)$.

The sampled tensor $\mathcal{U}$ includes $n_d$ subtensors that belong to  $\mathbb{R}^{n_1 \times n_2 \times \cdots \times n_{d-1} \times 1 }$ and let $\mathcal{Y}_i$ for $1 \leq i \leq n_d$ denote these $n_d$ subtensors. Define a binary valued tensor $\mathcal{\breve{Y}}_{i} \in \mathbb{R}^{n_1 \times n_2 \times \cdots \times n_{d-1} \times  k_i}$, where $k_i= N_{\Omega}(\mathcal{Y}_{i}) - r_{d-1}$ and its entries are described as the following. We can look at $\mathcal{\breve{Y}}_{i}$ as $k_i$ tensors each belongs to $\mathbb{R}^{n_1 \times n_2 \times \cdots \times n_{d-1} \times 1 }$. For each of the mentioned $k_i$ tensors in $\mathcal{\breve{Y}}_{i}$ we set the entries corresponding to the $r_{d-1}$ observed entries that are used to obtain $\mathcal{U}^{(d)}$  equal to $1$. For each of the other $k_i$ observed entries, we pick one of the $k_i$ tensors of $\mathcal{\breve{Y}}_{i}$ and set its corresponding entry (the same location as that specific observed entry) equal to $1$ and set the rest of the entries equal to $0$. In the case that $k_i=0$ we simply ignore $\mathcal{\breve{Y}}_{i}$, i.e., $\mathcal{\breve{Y}}_{i} = \emptyset$

By putting together all $n_d$ tensors in dimension $d$, we construct a binary valued tensor ${\breve{\Omega}} \in \mathbb{R}^{n_1 \times n_2 \times \cdots \times n_{d-1} \times K}$, where $K = \sum_{i=1}^{n_d} k_i = N_{\Omega}(\mathcal{U}) - r_{d-1}n_d$ and call it the {\bf constraint tensor}. Observe that each subtensor of the constraint tensor that belongs to $\mathbb{R}^{n_1 \times n_2 \times \cdots \times n_{d-1} \times 1 }$, i.e., each column of the $(d-1)$-th unfolding of ${\breve{\Omega}}$, includes exactly $r_{d-1}+1$ nonzero entries. The following example shows this procedure.

\begin{example}
{\rm Consider an example in which $d=3$ and $r_2=2$ and $\mathcal{U} \in \mathbb{R}^{3 \times 3 \times 3}$. Assume that $\Omega(x,y,z)=1$ if $(x,y,z) \in \mathcal{S}$ and $\Omega(x,y,z)=0$ otherwise, where 
\begin{eqnarray}
\mathcal{S} = \{(1,1,1),(1,2,1),(2,3,1), (3,3,1),(1,1,2),(2,1,2),(3,2,2),(1,3,3),(3,2,3)\}, \nonumber
\end{eqnarray}
represents the set of observed entries. Hence, observed entries $(1,1,1),(1,2,1),(2,3,1), (3,3,1)$ belong to $\mathcal{Y}_{1}$, observed entries $(1,1,2),(2,1,2),(3,2,2)$ belong to $\mathcal{Y}_{2}$, and observed entries $(1,3,3),(3,2,3)$ belong to $\mathcal{Y}_{3}$. As a result, $k_1 = 4-2 =2$, $k_2 =3-2=1$, and $k_3 =2-2=0$.

Also, assume that the entries that we use to obtain $\mathcal{U}^{(3)}$ given $\mathcal{U}^{(1)},\mathcal{U}^{(2)}$ are $(2,3,1)$, $(3,3,1)$, $(1,1,2)$, $(2,1,2)$, $(1,3,3)$ and $(3,2,3)$. Hence, $\mathcal{\breve{Y}}_{1} \in \mathbb{R}^{3 \times 3 \times 2 }$, $\mathcal{\breve{Y}}_{2} \in \mathbb{R}^{3 \times 3 \times 1 }$, and $\mathcal{\breve{Y}}_{3} = \emptyset$, and therefore the constraint tensor ${\breve{\Omega}}$ belongs to $\mathbb{R}^{3 \times 3 \times 3}$.

Note that $\mathcal{\breve{Y}}_{1}(2,3,1) = \mathcal{\breve{Y}}_{1}(2,3,2) = \mathcal{\breve{Y}}_{1}(3,3,1) = \mathcal{\breve{Y}}_{1}(3,3,2) =1$ (correspond to entries of $\mathcal{Y}_1$ that have been used to obtain $\mathcal{U}^{(3)}$), and also for the two other observed entries we have $\mathcal{\breve{Y}}_{1}(1,1,1) =1 $ (correspond to $\mathcal{U}(1,1,1)$) and $\mathcal{\breve{Y}}_{1}(1,2,2)=1$ (correspond to $\mathcal{U}(1,2,1)$) and the rest of the entries of $\mathcal{\breve{Y}}_{1}$ are equal to zero. Similarly, $\mathcal{\breve{Y}}_{2}(1,1,1) = \mathcal{\breve{Y}}_{2}(2,1,1) = \mathcal{\breve{Y}}_{2}(3,2,1)=1$ and the rest of the entries of $\mathcal{\breve{Y}}_{2}$ are equal to zero.

Then, ${\breve{\Omega}}(x,y,z)=1$ if $(x,y,z) \in \mathcal{S}^{\prime}$ and ${\breve{\Omega}}(x,y,z)=0$ otherwise, where 
\begin{eqnarray}
\mathcal{\breve{S}} = \{(1,1,1),(1,2,2),(2,3,1),(2,3,2),(3,3,1),(3,3,2),(1,1,3),(2,1,3),(3,2,3)\}. \nonumber
\end{eqnarray}  } {  \hfill  \qedsymbol}
\end{example}

Note that each subtensor of ${\breve{\Omega}}$ that belongs to $\mathbb{R}^{n_1 \times \dots \times n_{d-1} \times 1}$ represents one of the polynomials in $\mathcal{P}(\Omega)$ besides showing the polynomials that have been used to obtain $\mathcal{U}^{(d)}$. More specifically, consider a subtensor of ${\breve{\Omega}}$ that belongs to $\mathbb{R}^{n_1 \times \dots \times n_{d-1} \times 1}$ with $r_{d-1}+1$ nonzero entries. Observe that exactly $r_{d-1}$ of them correspond to the observed entries that have been used to obtain $\mathcal{U}^{(d)}$. Hence, this subtensor represents a polynomial after replacing entries of $\mathcal{U}^{(d)}$ by the expressions in terms of entries of $\mathcal{U}^{(1)},\dots ,\mathcal{U}^{(d-1)}$, i.e., a polynomial in $\mathcal{P}(\Omega)$.

\subsection{Algebraic Independence}\label{algebraind}

In Lemma \ref{finminpolned}, we obtained the required number of algebraically independent polynomials in $\mathcal{P}(\Omega)$ for finite completability, and therefore we can certify finite completability based on the maximum number of algebraically independent polynomials in $\mathcal{P}(\Omega)=\mathcal{P}({\breve{\Omega}})$. In this subsection, a sampling pattern on the constraint tensor is proposed to obtain the maximum number of algebraically independent polynomials in $\mathcal{P}({\breve{\Omega}})$ based on the structure of the nonzero entries of ${\breve{\Omega}}$.

\begin{definition}
Let ${\breve{\Omega}}^{\prime} \in \mathbb{R}^{n_1 \times n_2 \times \cdots \times n_{d-1} \times t}$ be a subtensor of the constraint tensor ${\breve{\Omega}}$. Let $m_i({\breve{\Omega}}^{\prime})$ denote the number of nonzero rows of ${\mathbf {\breve{\Omega}}}^{\prime}_{(i)}$. Also, let $\mathcal{P}({\breve{\Omega}}^{\prime})$ denote the set of polynomials that correspond to nonzero entries of ${\breve{\Omega}}^{\prime}$.
\end{definition}

Recall Facts $1$ and $2$ regarding the number of involved entries of components of the TT decomposition in a set of polynomials. However, according to Lemma \ref{pattern} some of the entries of $\mathcal{U}^{(i)}$'s are known, i.e., $(\mathbf{P}_{1},\dots,\mathbf{P}_{d-1})$ that satisfy properties (i) and (ii) in Definition \ref{defstrucpropTT}. Therefore, in order to find the number of variables (unknown entries of $\mathcal{U}^{(i)}$'s) in a set of polynomials, we should subtract the number of known entries in the corresponding pattern from the total number of involved entries.

For any subtensor ${\breve{\Omega}}^{\prime} \in \mathbb{R}^{n_1 \times n_2 \times \cdots \times n_{d-1} \times t}$ of the constraint tensor, the next lemma gives an upper bound on the number of algebraically independent polynomials in the set $\mathcal{P}({\breve{\Omega}}^{\prime})$. Recall that $\mathcal{P}({\breve{\Omega}}^{\prime})$ includes exactly $t$ polynomials.

\begin{lemma}\label{uppboundindppoly}
Suppose that Assumption $1$ holds. For any subtensor ${\breve{\Omega}}^{\prime} \in \mathbb{R}^{n_1 \times n_2 \times \cdots \times n_{d-1} \times t}$ of the constraint tensor, the maximum number of algebraically independent polynomials in $\mathcal{P}({\breve{\Omega}}^{\prime})$ is upper bounded
by
\begin{eqnarray}\label{numbalginprw}
\sum_{i=1}^{d-1} \left(r_{i-1}r_im_i({\breve{\Omega}}^{\prime}) - r_i^2\right)^{+}.
\end{eqnarray}
\end{lemma}

\begin{proof}
Recall Fact $1$ which states that any of the $t$ polynomials in $\mathcal{P}({\breve{\Omega}}^{\prime})$ involves exactly $r_{i-1}r_i$ entries of $\mathcal{U}^{(i)}$, $i=1,2,\dots,d$. Moreover, we use Fact $2$ in order to find the number of entries of tuple $(\mathcal{U}^{(1)},\dots ,\mathcal{U}^{(d-1)})$ that are involved in at least one of the polynomials in $\mathcal{P}({\breve{\Omega}}^{\prime})$. Note that according to Fact $2$, if an entry $(x_1,\dots,x_d)$ is observed such that $x_i=l$, then all $r_{i-1}r_i$ entries of the $l$-th row of the second (first) matricization of $\mathcal{U}^{(i)}$ are involved in the polynomial corresponding to this observed entry, $i=2,\dots,d-1$ ($i=1$). Hence, it is easily verified that the total number of involved entries of the tuple $(\mathcal{U}^{(1)},\dots ,\mathcal{U}^{(d-1)})$ in the $t$ polynomials in $\mathcal{P}({\breve{\Omega}}^{\prime})$ is $\sum_{i=1}^{d-1} r_{i-1}r_im_i({\breve{\Omega}}^{\prime}) $.

On the other hand, among the $\sum_{i=1}^{d-1} r_{i-1}r_im_i({\breve{\Omega}}^{\prime}) $ known entries corresponding to $(\mathbf{P}_{1},\dots,\mathbf{P}_{d-1})$ in TT decomposition, some of them are involved in polynomials of $\mathcal{P}({\breve{\Omega}}^{\prime})$. Recall that $(\mathbf{P}_{1},\dots,\mathbf{P}_{d-1})$ satisfy properties (i) and (ii) in Definition \ref{defstrucpropTT}. Among all TT decompositions, consider the one that has the maximum number of known entries that are involved in the polynomials in $\mathcal{P}({\breve{\Omega}}^{\prime})$. For $\mathcal{U}^{(i)}$, this number (the maximum number of known entries) is $\min\{r_i^2,r_{i-1}r_im_i({\breve{\Omega}}^{\prime})\}$. Hence, the number of variables that are involved in the set of polynomials $\mathcal{P}({\breve{\Omega}}^{\prime})$ is $\sum_{i=1}^{d-1} \left(r_{i-1}r_im_i({\breve{\Omega}}^{\prime}) - r_i^2\right)^{+}$. The proof is complete since using Fact $3$, the number of algebraically independent polynomials in a subset of polynomials of $\mathcal{P}({\breve{\Omega}}^{\prime})$ is at most equal to the total number of variables that are involved in the corresponding polynomials.
\end{proof}

We are interested in obtaining the maximum number of algebraically independent polynomials in $\mathcal{P}({\breve{\Omega}}^{\prime})$ as Lemma \ref{uppboundindppoly} only provides an upper bound. A subset of polynomials $\mathcal{P}({\breve{\Omega}}^{\prime})$ is minimally algebraically dependent if the polynomials in $\mathcal{P}({\breve{\Omega}}^{\prime})$ are algebraically dependent but polynomials in every of its proper subset are algebraically independent. The next lemma which is Lemma $3$ in \cite{ashraphijuo} will be used to determine if the polynomials in the set $\mathcal{P}({\breve{\Omega}}^{\prime})$ are algebraically dependent.

\begin{lemma}\label{mindepnumvar}
Suppose that Assumption $1$ holds. Suppose that ${\breve{\Omega}}^{\prime} \in \mathbb{R}^{n_1 \times n_2 \times \cdots \times n_{d-1} \times t}$ is a subtensor of the constraint tensor such that $\mathcal{P}({\breve{\Omega}}^{\prime})$ is minimally algebraically dependent. Then, for almost every $\mathcal{U}$, the number of variables that are involved in the set of polynomials $\mathcal{P}({\breve{\Omega}}^{\prime})$ is $t-1$.
\end{lemma}

The next lemma characterizes a relationship between the number of algebraically independent polynomials in $\mathcal{P}({\breve{\Omega}})$ and the structure of the nonzero entries of ${\breve{\Omega}}$.

\begin{lemma}\label{charnumindppoly}
Suppose that Assumption $1$ holds and consider a subtensor ${\breve{\Omega}}^{\prime } \in \mathbb{R}^{n_1 \times n_2 \times \cdots \times n_{d-1} \times t}$ of the constraint tensor ${\breve{\Omega}}$. The polynomials in the set $\mathcal{P}({\breve{\Omega}}^{\prime})$ are algebraically dependent if and only if \\ $ \sum_{i=1}^{d-1} \left(r_{i-1}r_im_i({\breve{\Omega}}^{\prime \prime}) - r_i^2\right)^{+} < t^{\prime}$ \ for some subtensor ${\breve{\Omega}}^{\prime \prime} \in \mathbb{R}^{n_1 \times n_2 \times \cdots \times n_{d-1} \times t^{\prime}}$ of ${\breve{\Omega}}^{\prime }$.
\end{lemma}

\begin{proof}
First assume that $ \sum_{i=1}^{d-1} \left(r_{i-1}r_im_i({\breve{\Omega}}^{\prime \prime}) - r_i^2\right)^{+} < t^{\prime}$ \ for some subtensor ${\breve{\Omega}}^{\prime \prime} \in \mathbb{R}^{n_1 \times n_2 \times \cdots \times n_{d-1} \times t^{\prime}}$ of the tensor ${\breve{\Omega}}^{\prime}$. Recall that $t^{\prime}$ is the number of polynomials in $\mathcal{P}({\breve{\Omega}}^{\prime})$. On the other hand, according to Lemma \ref{uppboundindppoly}, $ \sum_{i=1}^{d-1} \left(r_{i-1}r_im_i({\breve{\Omega}}^{\prime \prime}) - r_i^2\right)^{+}$ is the maximum number of algebraically independent polynomials, and therefore the polynomials in $\mathcal{P}({\breve{\Omega}}^{\prime \prime})$ are not algebraically independent.

Now, assume that the polynomials in set $\mathcal{P}({\breve{\Omega}}^{\prime})$ are algebraically dependent. Then, there exists a subset of the polynomials that are minimally algebraically dependent. According to Lemma \ref{mindepnumvar}, if ${\breve{\Omega}}^{\prime \prime } \in \mathbb{R}^{n_1 \times n_2 \times \cdots \times n_{d-1} \times t^{\prime}}$ is the corresponding subtensor to this minimally algebraically dependent set of polynomials, the number of variables that are involved in $\mathcal{P}({\breve{\Omega}}^{\prime \prime})=\{p_1,,p_2\dots,p_{t^{\prime}}\}$ is equal to $t^{\prime}-1$. On the other hand, $\sum_{i=1}^{d-1} \left(r_{i-1}r_im_i({\breve{\Omega}}^{\prime \prime}) - r_i^2\right)^{+}$ is the minimum possible number of involved variables in $\mathcal{P}({\breve{\Omega}}^{\prime \prime })$ since $ \sum_{i=1}^{d-1} \min\{r_i^2,r_{i-1}r_im_i({\breve{\Omega}}^{\prime \prime})\}$ is the maximum number of known entries that are involved in $\mathcal{P}({\breve{\Omega}}^{\prime \prime})$. Hence, we have $\sum_{i=1}^{d-1} \left(r_{i-1}r_im_i({\breve{\Omega}}^{\prime \prime}) - r_i^2\right)^{+} \leq t-1 $.
\end{proof}

Finally, the following theorem characterizes the necessary and sufficient condition on the sampling patterns for finite completability of the sampled tensor $\mathcal{U}$ given its TT rank.

\begin{theorem}\label{detconfinTT}
Suppose that Assumption $1$ holds. Then, for almost every $\mathcal{U}$, there are only finitely many tensors that fit in the sampled tensor $\mathcal{U}$, and have TT rank $(r_1,r_2,\dots,r_{d-1})$ if and only if the following two conditions hold: 

(i) there exists a subtensor ${\breve{\Omega}}^{\prime} \in \mathbb{R}^{n_1 \times n_2 \times \cdots \times n_{d-1} \times M}$ of the constraint tensor such that \ $ M = \sum_{i=1}^{d-1} r_{i-1}n_ir_i -\sum_{i=1}^{d-1} r_i^2 $, and 

(ii) for any $t \in \{1,\dots,M\}$ and any subtensor ${\breve{\Omega}}^{\prime \prime} \in \mathbb{R}^{n_1 \times n_2 \times \cdots \times n_{d-1} \times t}$ of the tensor ${\breve{\Omega}}^{\prime}$, the following inequality holds
\begin{eqnarray}\label{ineqp}
\sum_{i=1}^{d-1} \left(r_{i-1}r_im_i({\breve{\Omega}}^{\prime \prime}) - r_i^2\right)^{+} \geq t.
\end{eqnarray}
\end{theorem}

\begin{proof}
As a result of Lemma \ref{charnumindppoly}, the polynomials in $\mathcal{P}({\breve{\Omega}}^{\prime})$ are algebraically independent if and only if condition (ii) in the statement of the theorem holds. On the other hand, Lemma \ref{finminpolned} concludes that for almost every $\mathcal{U}$, there are finitely many completions of $\mathcal{U}$ if and only if there exist $\sum_{i=1}^{d-1} r_{i-1}n_ir_i -\sum_{i=1}^{d-1} r_i^2$ algebraically independent polynomials in $\mathcal{P}({\breve{\Omega}})$. Therefore, for almost every $\mathcal{U}$, there are finitely many completions of $\mathcal{U}$ if and only if conditions (i) and (ii) hold.
\end{proof}

\section{Probabilistic Conditions for Finite Completability}\label{secprob}

In this section, consider a $d$-way sampled tensor $\mathcal{U} \in \mathbb{R}^{\overbrace {n \times \dots \times n}^{d}}$ with TT-$\text{rank} (\mathcal{U})=(r_1,\dots,r_{d-1})$. Assume that the entries of $\mathcal{U}$ are independently sampled with probability $p$. Under a set of mild assumptions, we bound the sampling probability, or equivalently, the number of needed samples such that the corresponding constraint tensor satisfies conditions (i) and (ii) in the statement of Theorem \ref{detconfinTT} with high probability. In other words, satisfying the bound on the number of samples guarantees that the sampled tensor $\mathcal{U}$ is finitely completable with high probability. Assume that the entries of the tensor are sampled independently with probability $p$.

We note that this problem was considered for the matrix case in \cite{charact}. Hence, one may apply the result of this problem for matrix on each unfolding (since unfolding ranks are given), which is discussed in Section \ref{subsmatappda}. Then, in Section \ref{substenapash}, we will develop a combinatorial method in terms of the number of samples to verify if Theorem \ref{detconfinTT} holds.

\subsection{Unfolding Approach}\label{subsmatappda}

First, we restate Theorem $3$ in \cite{charact} which is the basis of the unfolding approach.

\begin{theorem}\label{thmmat}
Consider an $n \times N$ matrix with the given rank $r$ and let $0 < \epsilon < 1$ be given. Suppose $r \leq \frac{n}{6}$ and that each {\bf column} of the sampled matrix is observed in at least $l$ entries, distributed uniformly at random and independently across entries, where
\begin{eqnarray}\label{genmatrix}
l > \max\left\{12 \ \log \left( \frac{n}{\epsilon} \right) + 12, 2r\right\}. 
\end{eqnarray}
Also, assume that $ r(n-r) \leq N$. Then, with probability at least $1 - \epsilon$, the sampled matrix will be finitely completable.
\end{theorem}

Observe that in the case that $1< r< n-1$, the assumption $ r(n-r) \leq N$ results that $n < N$ which is very important to check when we apply this theorem. We can simply apply Theorem \ref{thmmat} to each unfolding of the sampled tensor, to obtain the following.

\begin{corollary}\label{matmethsamnum1}
Assume that $i \leq \frac{d-1}{2}$ and $1 < r_i \leq \frac{n^i}{6}$. Note that $ n^i \leq r_i n^{d-i}$ and $n^{d-i} > r_i(n^i-r_i)$ hold. Suppose that each {\bf column} of the $i$-th unfolding of the sampled tensor is observed in at least $l$ entries, distributed uniformly at random and independently across entries, where
\begin{eqnarray}\label{genmatrix1}
l > \max\left\{12 \ \log \left( \frac{n^i}{\epsilon} \right) + 12, 2r_i\right\}. 
\end{eqnarray}
Then, since $ n^i \leq r_i n^{d-i}$ and according to Theorem \ref{thmmat}, with probability at least $1 - \epsilon$, the sampled tensor (unfolding matrix) is finitely completable. This results in $n^{d-i} \max\left\{12 \ \log \left( \frac{n^i}{\epsilon} \right) + 12, 2r_i\right\}$ samples in total. 

Now, assume that $i \geq \frac{d+1}{2}$ and $1 < r_i \leq \frac{n^{d-i}}{6}$. Note that $ n^{d-i} \leq r_i n^{i}$ and $n^{i} > r_i(n^{d-i}-r_i)$ hold. Suppose that each {\bf row} of the $i$-th unfolding of the sampled tensor is observed in at least $l$ entries, distributed uniformly at random and independently across entries, where
\begin{eqnarray}\label{genmatrix2}
l > \max\left\{12 \ \log \left( \frac{n^{d-i}}{\epsilon} \right) + 12, 2r_i\right\}. 
\end{eqnarray}
Then, since $ n^{d-i} \leq r_i n^{i}$ and according to Theorem \ref{thmmat}, with probability at least $1 - \epsilon$, the sampled tensor (unfolding matrix) is finitely completable. This results in $n^{i} \max\left\{12 \ \log \left( \frac{n^{d-i}}{\epsilon} \right) + 12, 2r_i\right\}$ samples in total.
\end{corollary}

Assume that $i = \frac{d}{2}$ and $1 < r_i$. Then, as the $i$-th unfolding of the sampled tensor is an $n^i \times n^i$ matrix, we can simply verify that Theorem \ref{thmmat} is not applicable due the assumption $ r(n-r) \leq N$.

\begin{remark}\label{remsam0}
Consider a tensor $\mathcal{U}$ that satisfies $1 < r_i \leq \frac{n^i}{6}$ for $i \leq \frac{d-1}{2}$ and $1 < r_i \leq \frac{n^{d-i}}{6}$ for $i \geq \frac{d+1}{2}$. According to Corollary \ref{matmethsamnum1}, the sampled tensor $\mathcal{U}$ requires more than 
\begin{eqnarray}\label{rremsam0}
n^{\lceil \frac{d+1}{2} \rceil } \max\left\{12 \ \log \left( \frac{n^{\lfloor \frac{d-1}{2} \rfloor }}{\epsilon} \right) + 12, 2r_{\lfloor \frac{d-1}{2} \rfloor }\right\}
\end{eqnarray}
samples to be finitely completable with probability at least $1-\epsilon$.
\end{remark}

\subsection{TT Approach}\label{substenapash}

In the second approach, instead of using Theorem \ref{thmmat} which is taken from \cite{charact}, we are interested in finding the number of sampled entries which ensures conditions (i) and (ii) in the statement of Theorem \ref{detconfinTT} to hold with high probability. The following lemma is Lemma $5$ in \cite{ashraphijuo2} and will be used later to obtain Lemma \ref{lemman2}.

\begin{lemma}\label{genlem}
Assume that $r^{\prime} \leq \frac{n}{6}$ and also each column of $\mathbf{\Omega}_{(1)}$ (first matricization of $\Omega$) includes at least $l$ nonzero entries, where 
\begin{eqnarray}\label{genminl1}
l > \max\left\{9 \ \log \left( \frac{n}{\epsilon} \right) + 3 \ \log \left( \frac{k}{\epsilon} \right) + 6, 2r^{\prime}\right\}. 
\end{eqnarray}
Let $\mathbf{\Omega}^{\prime}_{(1)}$ be an arbitrary set of $n -r^{\prime}$ columns of $\mathbf{\Omega}_{(1)}$. Then, with probability at least $1-\frac{\epsilon}{k}$, every subset $\mathbf{\Omega}^{\prime \prime}_{(1)}$ of columns of $\mathbf{\Omega}^{\prime}_{(1)}$ satisfies 
\begin{eqnarray}\label{genproper1}
m_{1}({\Omega}^{\prime \prime}) - r^{\prime} \geq t,
\end{eqnarray}
where $t$ is the number of columns of $\mathbf{\Omega}^{\prime \prime}_{(1)}$ and $m_{1}({\Omega}^{\prime \prime})$ is the number of nonzero rows of $\mathbf{\Omega}^{\prime \prime}_{(1)}$.
\end{lemma}


The following lemma provides a bound on the number of sampled entries at each column of the $j$-th unfolding of the sampled tensor such that the $i$-th matricization of the subtensor corresponding to a columns of the $j$-th unfolding includes more than the RHS of \eqref{genminl1} observed entries of $\Omega$ with different values of the $i$-th coordinate.

\begin{lemma}\label{lemman2p}
Assume that $r^{\prime} \leq \frac{n}{6}$ and also let $j \in \{1,2,\dots,d-1\}$ be a fixed number. Consider an arbitrary set $\widetilde{\mathbf{\Omega}}^{\prime}_{(j)}$ of $n -r^{\prime}$ columns of $\widetilde{\mathbf{\Omega}}_{(j)}$ ($j$-th unfolding of $\Omega$). Assume that $n > \max \{ 200 ,\sum_{k=1}^{d-1} r_{k-1}r_{k}\}$, and also each column of $\widetilde{\mathbf{\Omega}}_{(j)}$ includes at least $l$ nonzero entries, where 
\begin{eqnarray}\label{minlforset2pn}
l > \max\left\{27 \ \log \left( \frac{n}{\epsilon} \right) + 9 \ \log \left( \frac{2r}{\epsilon} \right) + 18, 6r^{\prime}\right\},
\end{eqnarray}
where $r \leq \sum_{k=1}^{d-1} r_{k-1}r_{k}$ (recall that $r_0=r_d=1$). Then, with probability at least $1-\frac{\epsilon}{2r}$, each column of $\widetilde{\mathbf{\Omega}}^{\prime}_{(j)}$ includes more than $l_0 \triangleq \max\left\{9 \ \log \left( \frac{n}{\epsilon} \right) + 3 \ \log \left( \frac{2r}{\epsilon} \right) + 6, 2r^{\prime}\right\}$ observed entries of $\Omega$ with different values of the $i$-th coordinate, i.e., the $i$-th matricization of the tensor $\Omega^{\prime}$ that corresponds to $\widetilde{\mathbf{\Omega}}^{\prime}_{(j)}$ includes more than $l_0$ nonzero rows, $1 \leq i \leq j$.
\end{lemma}

\begin{proof}
Each column of $\widetilde{\mathbf{\Omega}}^{\prime}_{(j)}$ includes $n^j$ entries and they can be represented by $(x_1,\dots,x_j)$ for $1 \leq x_k \leq n$ and $1 \leq k \leq j$, where $x_k$ denotes the $k$-th coordinate of the corresponding entry. Let $P(\zeta)$ be the probability that at least one of the columns of $\widetilde{\mathbf{\Omega}}^{\prime}_{(j)}$ includes at most $l_0$ observed entries of $\Omega$ with different values of the $i$-th coordinate. Also, let $P(\zeta_s)$ denote the probability that the $s$-th column of $\widetilde{\mathbf{\Omega}}^{\prime}_{(j)}$ includes at most $l_0$ observed entries of $\Omega$ with different values of the $i$-th coordinate, $1 \leq s \leq n-r^{\prime}$. Then, we have $P(\zeta) \leq (n-r^{\prime}) P(\zeta_1)$.

By assumption, each column of $\widetilde{\mathbf{\Omega}}^{\prime}_{(j)}$ includes more than $3l_0$ observed entries. In the case that the first column of $\widetilde{\mathbf{\Omega}}^{\prime}_{(j)}$ includes at most $l_0$ observed entries of $\Omega$ with different values of the $i$-th coordinate, we conclude the set of $i$-th coordinates of all observed entries of this column (which are more than $3l_0$ entries) belong to a set with at most $l_0$ numbers. As it is assumed to have the uniform random sampling, we have
\begin{eqnarray}\label{loginq}
P(\zeta_1) \leq {n \choose l_0} \left( \frac{l_0}{n} \right)^{3l_0}. 
\end{eqnarray}
Furthermore, we have 
\begin{eqnarray}\label{ineup}
{ n\choose l_0} = \frac{n(n-1)\dots(n-l_0+1)}{l_0!} \leq \frac{n^{l_0}}{l_0!} \leq \left(\frac{ne}{l_0}\right)^{l_0},
\end{eqnarray}
where the last inequality holds since $e^{l_0} = \sum_{k=0}^{\infty} \frac{{l_0}^k}{k!} \geq \frac{{l_0}^{l_0}}{l_0!}$. Having \eqref{loginq} and \eqref{ineup}, we can conclude
\begin{eqnarray}\label{loginq2}
P(\zeta_1) \leq e^{l_0} \left( \frac{l_0}{n} \right)^{2l_0} = \left( \frac{e^{\frac{1}{2}}l_0}{n} \right)^{2l_0},
\end{eqnarray}
and therefore
\begin{eqnarray}\label{loginq3}
\log \left( P(\zeta) \right) \leq \log \left( (n-r^{\prime}) P(\zeta_1) \right) \stackrel{(a)}{<}  2l_0 \left( \frac{1}{2}+ \log (l_0) - \log(n) \right) + \log(n) \nonumber \\
 \stackrel{(b)}{\leq} 2l_0 \left( \frac{1}{2}+ \log (l_0) - \log(n) \right) + \frac{1}{9} l_0 = 2l_0 \left(  \frac{13}{18} + \log (l_0) - \log(n) \right) - \frac{l_0}{3},
\end{eqnarray}
where $(a)$ follows from the fact that $\log (n-r^{\prime}) < \log (n)$ and $(b)$ follows from $l_0 \geq 9 \log (n) - 9 \log (\epsilon) \geq 9 \log (n)$ which is easy to verify having the definition of $l_0$. On the other hand, we have
\begin{eqnarray}\label{loginq4}
- \frac{l_0}{3} \leq -3 \ \log \left( \frac{n}{\epsilon} \right) - \ \log \left( \frac{2r}{\epsilon} \right) -2 = 4 \log (\epsilon) - 3 \log (n) -  \log (2r) -2 \nonumber \\ 
\stackrel{(c)}{<} \log (\epsilon) -  \log (2r) = \log \left(\frac{\epsilon}{2r} \right),
\end{eqnarray}
where $(c)$ follows from $3 \log(\epsilon) - 3 \log (n) - 2 <0$ since $\log (\epsilon) < 0 < \log (n)$. Moreover, for the term $\frac{13}{18} + \log (l_0) - \log(n)$, there are following two possibilities:

(i) {\it $l_0=2r^{\prime}$}: {\it We conclude $\frac{13}{18} + \log (l_0) - \log(n) < \log (2.06) + \log (2r^{\prime}) - \log(n) =  \log \left( \frac{4.12 \ r^{\prime}}{n} \right) < 0$, where the last inequality is a simple result of the assumption $r^{\prime} \leq \frac{n}{6}$.}

(ii) {\it $l_0=9 \ \log \left( \frac{n}{\epsilon} \right) + 3 \ \log \left( \frac{2r}{\epsilon} \right) + 6$}: {\it Recall that $r \leq \sum_{k=1}^{d-1} r_{k-1}r_{k} < n$, and therefore $l_0 \leq 12 \log (n) +  6 + 3 \log(2)$. Then, having the assumption $200 < n$, we simply conclude $\frac{13}{18} + \log (l_0) - \log(n) \leq \frac{13}{18} + \log (12 \log (n) +  6 + 3 \log(2)) - \log(n)   < 0$.}

Therefore, the assumptions $ \max \{200 , \sum_{k=1}^{d-1} r_{k-1}r_{k}\} < n$ and $r^{\prime} \leq \frac{n}{6}$ result in
\begin{eqnarray}\label{loginq5}
\frac{13}{18} + \log (l_0) - \log(n) \leq 0.
\end{eqnarray}
Having \eqref{loginq3}, \eqref{loginq4}, and \eqref{loginq5} result  that $\log \left( P(\zeta) \right) < \log \left(\frac{\epsilon}{2r} \right)$, and the proof is complete.
\end{proof}

The following lemma exploits Lemma \ref{genlem} and Lemma \ref{lemman2p} to provide a bound on the number of sampled entries at each column of the $j$-th unfolding of the sampled tensor such that the $i$-th matricization of the subtensor corresponding to a subset of columns of the $j$-th unfolding satisfies the property in the statement of Lemma \ref{genlem} with high probability.

\begin{lemma}\label{lemman2}
Let $j \in \{1,2,\dots,d-1\}$ be a fixed number. Assume that $r_{i}^{\prime} \leq \frac{n}{6}$, where $i \in \{1,\dots,j\}$. Consider an arbitrary set $\widetilde{\mathbf{\Omega}}^{\prime}_{(j)}$ of $n -r_i^{\prime}$ columns of $\widetilde{\mathbf{\Omega}}_{(j)}$. Assume that $n > \max \{200 , \sum_{k=1}^{d-1} r_{k-1}r_{k}\} $, and also each column of $\widetilde{\mathbf{\Omega}}_{(j)}$ includes at least $l$ nonzero entries, where 
\begin{eqnarray}\label{minlforset2pne}
l > \max\left\{27 \ \log \left( \frac{n}{\epsilon} \right) + 9 \ \log \left( \frac{2r}{\epsilon} \right) + 18, 6r_i^{\prime}\right\},
\end{eqnarray}
where $r \leq \sum_{k=1}^{d-1} r_{k-1}r_{k}$ (recall that $r_0=r_d=1$). Then, with probability  at least $1-\frac{\epsilon}{r}$, every subset $\widetilde{\mathbf{\Omega}}^{\prime \prime}_{(j)}$ of columns of $\widetilde{\mathbf{\Omega}}^{\prime}_{(j)}$ satisfies 
\begin{eqnarray}\label{proper2}
m_{i}({\Omega}^{\prime \prime}) - r_i^{\prime} \geq t,
\end{eqnarray}
where $t$ is the number of columns of $\widetilde{\mathbf{\Omega}}^{\prime \prime}_{(j)}$ and $\Omega^{\prime \prime}$ is the corresponding tensor such that $\widetilde{\mathbf{\Omega}}^{\prime \prime}_{(j)}$ is the $j$-th unfolding of $\Omega^{\prime \prime}$.
\end{lemma}

\begin{proof}
Each column of $\widetilde{\mathbf{\Omega}}_{(j)}$ includes $n^j$ entries and they can be represented by $(x_1,\dots,x_j)$ for $1 \leq x_k \leq n$ and $1 \leq k \leq j$, where $x_k$ denotes the $k$-th coordinate of the corresponding entry. According to Lemma \ref{lemman2p}, with probability  at least $1-\frac{\epsilon}{2r}$, each column of $\widetilde{\mathbf{\Omega}}_{(j)}$ includes more than $\max\left\{9 \ \log \left( \frac{n}{\epsilon} \right) + 3 \ \log \left( \frac{2r}{\epsilon} \right) + 6, 2r^{\prime}\right\}$ observed entries with different values of the $i$-th coordinate. Therefore, according to Lemma \ref{genlem}, with probability  at least $(1-\frac{\epsilon}{2r})^2$, every subset $\widetilde{\mathbf{\Omega}}^{\prime \prime}_{(j)}$ of columns of $\widetilde{\mathbf{\Omega}}^{\prime}_{(j)}$ satisfies \eqref{proper2}. The proof is complete as $(1-\frac{\epsilon}{2r})^2 \geq 1-\frac{\epsilon}{r}$.
\end{proof}


The following lemma is taken from \cite[Lemma $8$]{ashraphijuo} which will be used to obtain Lemma \ref{lemman3}. This lemma states that if the property in Lemma \ref{genlem} holds for the sampling pattern $\Omega$, it will be satisfied for $\breve{\Omega}$ as well.

\begin{lemma}\label{Omega}
Let $r^{\prime}$ be a given nonnegative integer and $1 \leq i \leq j \leq d-1$. Assume that there exists an $n^j \times (n -r^{\prime})$ matrix $\widetilde{\mathbf{\Omega}}^{\prime}_{(j)}$ composed  of $n-r^{\prime}$ columns of $\widetilde{\mathbf{\Omega}}_{(j)}$ such that each column of $\widetilde{\mathbf{\Omega}}^{\prime}_{(j)}$ includes at least $r^{\prime}+1$ nonzero entries and satisfies the following property:
\begin{itemize}
\item Denote an $n^j \times t$  matrix (for any $1 \leq t \leq n-r^{\prime}$) composed of any $t$ columns of $\widetilde{\mathbf{\Omega}}^{\prime}_{(j)}$ by $\widetilde{\mathbf{\Omega}}^{\prime \prime}_{(j)}$. Then 
\begin{eqnarray}\label{proper233}
m_{i}({\Omega}^{\prime \prime}) -r^{\prime}  \geq t.
\end{eqnarray}
\end{itemize}
Then, there exists an $n^j \times (n -r^{\prime})$ matrix $\widetilde{\mathbf{\breve{\Omega}}}^{\prime}_{(j)}$ such that: each column has exactly $r^{\prime}+1$ entries equal to one, and if $\widetilde{\mathbf{\breve{\Omega}}}^{\prime}_{(j)}(x,y)=1$ then we have $\widetilde{\mathbf{\Omega}}^{\prime}_{(j)}(x,y)=1$. Moreover, $\widetilde{\mathbf{\breve{\Omega}}}^{\prime}_{(j)}$ satisfies the above-mentioned property.
\end{lemma}


\begin{lemma}\label{lemman3}
Assume that $1 \leq i \leq j \leq d-1$ and consider $r^{\prime}$ disjoint sets $\widetilde{\mathbf{\Omega}}^{{\prime}}_{{(j)}_k}$, each with $n -r^{\prime}_i$ columns of $\widetilde{\mathbf{\Omega}}_{(j)}$ for $1 \leq k \leq r^{\prime}$, where $r^{\prime}_i \leq \frac{n}{6}$ and $r^{\prime} \leq r \leq \sum_{k=1}^{d-1} r_{k-1}r_{k}$. Let $\widetilde{\mathbf{\Omega}}^{{\prime}}_{{(j)}}$ denote the union of all $r^{\prime}$ sets of columns $\widetilde{\mathbf{\Omega}}^{{\prime}}_{{(j)}_k}$'s, and therefore it includes $r^{\prime}(n-r^{\prime}_i)$ columns. Assume that $n > \max \{ 200 , \sum_{k=1}^{d-1} r_{k-1}r_{k} \} $, and also each column of $\widetilde{\mathbf{\Omega}}_{(j)}$ includes at least $l$ nonzero entries, where 
\begin{eqnarray}\label{minlforset3}
l > \max\left\{27 \ \log \left( \frac{n}{\epsilon} \right) + 9 \ \log \left( \frac{2r}{\epsilon} \right) + 18, 6r_i^{\prime}\right\}. 
\end{eqnarray}
Then, there exists an $n^j \times r^{\prime}(n -r^{\prime}_i)$ matrix $\widetilde{\breve{\mathbf{\Omega}}}^{\prime}_{(j)}$ such that each column has exactly $r^{\prime}_i+1$ entries equal to one, and if $\widetilde{\breve{\mathbf{\Omega}}}^{\prime}_{(j)}(x,y)=1$ then we have $\widetilde{\mathbf{\Omega}}^{\prime}_{(j)}(x,y)=1$ and also it satisfies the following property: with probability at least $1-\frac{\epsilon r^{\prime}}{r}$, every subset $\widetilde{\breve{\mathbf{\Omega}}}^{\prime \prime}_{(j)}$ of columns of $\widetilde{\breve{\mathbf{\Omega}}}^{\prime}_{(j)}$ satisfies the following inequality
\begin{eqnarray}\label{proper3}
r^{\prime} \left(m_{i}(\breve{{\Omega}}^{\prime \prime}) -r^{\prime}_i \right) \geq t,
\end{eqnarray}
where $t$ is the number of columns of $\widetilde{\breve{\mathbf{\Omega}}}^{\prime \prime}_{(j)}$ and $\breve{\Omega}^{\prime \prime}$ is the corresponding tensor such that $\widetilde{\breve{\mathbf{\Omega}}}^{\prime \prime}_{(j)}$ is the $j$-th unfolding of $\breve{\Omega}^{\prime \prime}$.
\end{lemma}

\begin{proof}
Consider any subset $\widetilde{{\mathbf{\Omega}}}^{\prime \prime}_{{(j)}_k}$ of columns of $\widetilde{{\mathbf{\Omega}}}^{ \prime}_{{(j)}_k}$ and consider its corresponding tensor ${\Omega}^{\prime \prime}_k$ such that the $j$-th unfolding of ${\Omega}^{\prime \prime}_j$ is $\widetilde{{\mathbf{\Omega}}}^{\prime \prime}_{{(j)}_k}$. First of all, according to Lemma \ref{lemman2}, ${\Omega}^{\prime \prime}_k$ satisfies the following inequality with probability at least $1-\frac{\epsilon}{r}$
\begin{eqnarray}
m_{i}({\Omega}^{\prime \prime}_k) - r_i^{\prime} \geq t_k, \label{proper5}
\end{eqnarray}
where $t_k$ is the number of columns of $\widetilde{{\mathbf{\Omega}}}^{\prime \prime}_{{(j)}_k}$.

According to Lemma \ref{Omega}, there exists an $n^j \times (n -r_i^{\prime})$ matrix $\widetilde{\breve{\mathbf{\Omega}}}^{\prime}_{{(j)}_k}$ such that each column has exactly $r_i^{\prime}+1$ entries equal to one, and if $\widetilde{\breve{\mathbf{\Omega}}}^{\prime}_{{(j)}_k}(x,y)=1$ then we have $\widetilde{{\mathbf{\Omega}}}^{{\prime}}_{{(j)}_{k}}(x,y)=1$ and also it satisfies the following property: with probability at least $1-\frac{\epsilon}{r}$, every subset $\widetilde{\breve{\mathbf{\Omega}}}^{\prime \prime}_{{(j)}_k}$ of columns of $\widetilde{\breve{\mathbf{\Omega}}}^{\prime}_{{(j)}_k}$ satisfies \eqref{proper5}. Define the union of the columns of $\widetilde{\breve{\mathbf{\Omega}}}^{\prime}_{{(j)}_k}$'s as $\widetilde{\breve{\mathbf{\Omega}}}^{\prime}_{{(j)}} = [\widetilde{\breve{\mathbf{\Omega}}}^{\prime}_{{(j)}_1}|\widetilde{\breve{\mathbf{\Omega}}}^{\prime}_{{(j)}_2}|\dots|\widetilde{\breve{\mathbf{\Omega}}}^{\prime}_{{(j)}_{r^{\prime}}}]$. In order to complete the proof it suffices to show that with probability at least $1-\epsilon$, the tensor $\breve{{\Omega}}^{\prime \prime}$ corresponding to any subset $\widetilde{\breve{\mathbf{\Omega}}}^{\prime \prime}_{{(j)}}$ of columns of $\widetilde{\breve{\mathbf{\Omega}}}^{\prime}_{{(j)}}$ satisfies \eqref{proper3}.

Let $\widetilde{\breve{\mathbf{\Omega}}}^{\prime \prime}_{{(j)}_k}$ denote those columns of $\widetilde{\breve{\mathbf{\Omega}}}^{\prime \prime}_{{(j)}}$ that belong to $\widetilde{\breve{\mathbf{\Omega}}}^{\prime}_{{(j)}_k}$ and define $s_k$ as the number of columns of $\widetilde{\breve{\mathbf{\Omega}}}^{\prime \prime}_{{(j)}_k}$, $1 \leq k \leq r^{\prime}$, and define $s$ as the number of columns of $\widetilde{\breve{\mathbf{\Omega}}}^{\prime \prime}_{{(j)}}$. Without loss of generality, assume that $s_1 \geq s_2 \geq \dots \geq s_{r^{\prime}}$. Also, assume that all ${\Omega}^{\prime \prime}_k$'s satisfy \eqref{proper5}. Hence, we have
\begin{eqnarray}\label{proper7} 
s = \sum_{k=1}^{r^{\prime}} s_k \leq r^{\prime} s_1 \leq r^{\prime} \left( m_{i}({\breve{\Omega}}^{\prime \prime}_{1}) -r_i^{\prime} \right) \leq r^{\prime} \left(m_{i}(\breve{\Omega}^{\prime \prime}) -r_i^{\prime} \right).
\end{eqnarray}
Observe that each ${\Omega}^{\prime \prime}_k$ satisfies \eqref{proper5} with probability at least $1-\frac{\epsilon}{r}$. Therefore, all ${\Omega}^{\prime \prime}_k$'s ($1 \leq k \leq r^{\prime}$) satisfy \eqref{proper5} with probability at least $1-\frac{\epsilon r^{\prime}}{r}$.
\end{proof}

Finally, the following theorem exploits Lemma \ref{lemman3} and Theorem \ref{detconfinTT} to obtain a bound on the number of sampled entries to ensure finite completability of the sampled tensor, with high probability.

\begin{theorem}\label{mainthsam}
Define $m= \sum_{k=1}^{d-2} r_{k-1}r_{k}$, $M = n \sum_{k=1}^{d-2} r_{k-1}r_k -\sum_{k=1}^{d-2} r_k^2$ and $r^{\prime}= \max \left\{  \frac{r_1}{r_0}  , \dots,  \frac{r_{d-2}}{r_{d-3}} \right\}$. Assume that $n > \max\{m,200\} $ and $r^{\prime} \leq \min\{\frac{n}{6},  r_{d-2}\}$ hold. Moreover, assume that each column of $\widetilde{\mathbf{\Omega}}_{(d-2)}$ includes at least $l$ nonzero entries, where 
\begin{eqnarray}\label{minlforset3th}
l > \max\left\{27 \ \log \left( \frac{n}{\epsilon} \right) + 9 \ \log \left( \frac{2M}{\epsilon} \right) + 18, 6r_{d-2}\right\}. 
\end{eqnarray}
Then, with probability at least $1-\epsilon$, for almost every $\mathcal{U} \in \mathbb{R}^{\overbrace {n \times \dots \times n}^{d}}$, there exist only finitely many completions of the sampled tensor $\mathcal{U}$ with separation rank $(r_1,r_2,\dots,r_{d-1})$.
\end{theorem}

\begin{proof}
Define the $(d-1)$-way tensor $\mathcal{U}^{\prime} \in \mathbb{R}^{\footnotesize {\overbrace {n \times \dots \times n}^{d-2}} \times {n}^2}$ which is obtained through merging the $(d-1)$-th and $d$-th dimensions of the tensor $\mathcal{U}$. Observe that the finiteness of the number of completions of the tensor $\mathcal{U}^{\prime}$ with rank vector $(r_1,r_2,\dots,r_{d-2})$  ensures the finiteness of the number of completions of the tensor $\mathcal{U}$ with rank vector $(r_1,r_2,\dots,r_{d-1})$. According to Theorem \ref{detconfinTT}, it suffices to show that with probability at least $1-\epsilon$, conditions (i) and (ii) in the statement of Theorem \ref{detconfinTT} hold for the tensor $\mathcal{U}^{\prime}$ with rank vector $(r_1,r_2,\dots,r_{d-2})$. 

Note that the assumption $m<n$ results that $M<n^2$, and therefore $\widetilde{\mathbf{\Omega}}_{(d-2)}$ has least $ M$ columns. Hence, for any $1 \leq i\leq d-2$ we can choose $r_{i-1}r_in-r_i^2$ arbitrary columns of $\widetilde{\mathbf{\Omega}}_{(d-2)}$ and denote it by $\widetilde{\mathbf{\Omega}}^{\prime}_{{(d-2)}_i}$ such that $\widetilde{\mathbf{\Omega}}^{\prime}_{{(d-2)}_i}$'s are disjoint sets of columns. Define $r_i^{\prime} = \lceil \frac{r_i}{r_{i-1}} \rceil$ and note that the assumption $r^{\prime} \leq r_{d-2}$ results that $r_i^{\prime} \leq r_{d-2}$. As a result, the assumption $r^{\prime}_i \leq r_{d-2}$ and \eqref{minlforset3th} results that $l > \max\left\{27 \ \log \left( \frac{n}{\epsilon} \right) + 9 \ \log \left( \frac{2M}{\epsilon} \right) + 18, 6r_{i}^{\prime}\right\}$. Therefore, according to Lemma \ref{lemman3}, there exists a matrix $\widetilde{\breve{\mathbf{\Omega}}}^{\prime}_{{(d-2)}_i}$ with $n^{d-2}$ rows and $  r_{i-1}r_i(n -\frac{r_i}{r_{i-1}}) = r_{i-1}r_in-r_i^2$ columns such that: each column has exactly $r^{\prime}_i+1$ entries equal to one, and if $\widetilde{\breve{\mathbf{\Omega}}}^{\prime}_{{(d-2)}_i}(x,y)=1$ then we have $\widetilde{\mathbf{\Omega}}^{\prime}_{{(d-2)}_i}(x,y)=1$ and also it satisfies the following property: with probability at least $1 - \frac{\epsilon r_{i-1} r_i}{M}$, every subset $\widetilde{\breve{\mathbf{\Omega}}}^{\prime \prime}_{{(d-2)}_i}$ of columns of $\widetilde{\breve{\mathbf{\Omega}}}^{\prime}_{{(d-2)}_i}$ satisfies the following
\begin{eqnarray}\label{proper3la}
r_{i-1}r_i m_{i}(\breve{{\Omega}}^{\prime \prime}_i) -r_i^2 \geq t,
\end{eqnarray}
where $t$ is the number of columns of $\widetilde{\breve{\mathbf{\Omega}}}^{\prime \prime}_{{(d-2)}_i}$ and $\breve{\Omega}^{\prime \prime}_i$ is the corresponding tensor such that $\widetilde{\breve{\mathbf{\Omega}}}^{\prime \prime}_{{(d-2)}_i}$ is the $(d-2)$-th unfolding of $\breve{\Omega}^{\prime \prime}_i$. Moreover, as we have $r_i^{\prime} \leq r_{d-2}$, by changing $r_{d-2} - r_i^{\prime}$ entries from zero to one at each column, we can assume that $\widetilde{\breve{\mathbf{\Omega}}}^{\prime}_{{(d-2)}_i}$ has exactly $r_{i-1}r_in-r_i^2$ columns of the $(d-2)$-th unfolding of the constraint tensor $\mathbf{\breve{\Omega}}$ and satisfies the above properties. Let ${\breve{\Omega}}^{\prime}_i$ denote the subtensor of the constraint tensor corresponding to $\widetilde{\breve{\mathbf{\Omega}}}^{\prime}_{{(d-2)}_i}$.

Let $\widetilde{\breve{\mathbf{\Omega}}}^{\prime}_{{(d-2)}}=[\widetilde{\breve{\mathbf{\Omega}}}^{\prime}_{{(d-2)}_1}| \dots | \widetilde{\breve{\mathbf{\Omega}}}^{\prime}_{{(d-2)}_{d-2}}]$ denote the union of $\widetilde{\breve{\mathbf{\Omega}}}^{\prime}_{{(d-2)}_i}$'s and ${\breve{\Omega}}^{\prime}$ denote its corresponding subtensor of the constraint tensor. Hence, ${\breve{\Omega}}^{\prime}$ satisfies condition (i) in the statement of Theorem \ref{detconfinTT} for tensor $\mathcal{U}^{\prime}$ with rank vector $(r_1,r_2,\dots,r_{d-2})$ since $\widetilde{\breve{\mathbf{\Omega}}}^{\prime}_{{(d-2)}}$ has $\sum_{k=1}^{d-2} r_{k-1}r_kn -\sum_{k=1}^{d-2} r_k^2$ columns. Furthermore, with probability at least $1- \frac{\epsilon}{M} \sum_{k=1}^{d-2} r_{k-1}r_k = 1 - \epsilon$, any subtensor ${\breve{\Omega}}^{\prime \prime} \in \mathbb{R}^{\footnotesize {\overbrace {n \times \dots \times n}^{d-2}} \times t}$ of tensor ${\breve{\Omega}}^{\prime}$ satisfies 
\begin{eqnarray}\label{proper4la}
\sum_{i=1}^{d-2} \left( r_{i-1}r_i m_{i}(\breve{{\Omega}}^{\prime \prime}) -r_i^2 \right)^{+} \geq \sum_{i=1}^{d-2} \left( r_{i-1}r_i m_{i}(\breve{{\Omega}}^{\prime \prime}_i) -r_i^2 \right)^{+} \geq \sum_{i=1}^{d-2} t_i = t,
\end{eqnarray}
where  ${\breve{\Omega}}^{\prime \prime}_i \in \mathbb{R}^{\footnotesize {\overbrace {n \times \dots \times n}^{d-2}} \times t_i}$ are such that ${\breve{\Omega}}^{\prime \prime}=[{\breve{\Omega}}^{\prime \prime}_1| \dots | {\breve{\Omega}}^{\prime \prime}_{d-2}]$ and ${\breve{\Omega}}^{\prime \prime}_i$ is a subtensor of ${\breve{\Omega}}^{\prime}_i$, $1 \leq i \leq d-2$. The proof is complete as condition (ii) in the statement of Theorem \ref{detconfinTT} holds.
\end{proof}

\begin{remark}\label{remsam}
A tensor $\mathcal{U}$ that satisfies the properties in the statement of Theorem \ref{mainthsam} requires 
\begin{eqnarray}\label{rremsam}
n^2 \max\left\{27 \ \log \left( \frac{n}{\epsilon} \right) + 9 \ \log \left( \frac{2M}{\epsilon} \right) + 18, 6r_{d-2}\right\}
\end{eqnarray}
samples to be finitely completable with probability at least $1-\epsilon$ since $\widetilde{\mathbf{\Omega}}_{(d-2)}$ has $n^2$ columns, with $M = n \sum_{k=1}^{d-2} r_{k-1}r_k -\sum_{k=1}^{d-2} r_k^2$, in contrast to the number of samples required by the unfolding approach given in Remark \ref{remsam0}.
\end{remark}

The following lemma is taken from \cite{ashraphijuo} and is used in Lemma \ref{probsamTTfin} to derive a lower bound on the sampling probability that results \eqref{minlforset3th} with high probability.

\begin{lemma}\label{azumares}
Consider a vector with $n$ entries where each entry is observed with  probability  $p$  independently from the other entries. If $p > p^{\prime} = \frac{k}{n} + \frac{1}{\sqrt[4]{n}}$, then with probability  at least $\left(1-\exp(-\frac{\sqrt{n}}{2})\right)$, more than $k$ entries are observed.
\end{lemma}

\begin{lemma}\label{probsamTTfin}
Define $m= \sum_{k=1}^{d-2} r_{k-1}r_{k}$, $M = n \sum_{k=1}^{d-2} r_{k-1}r_k -\sum_{k=1}^{d-2} r_k^2$ and $r^{\prime}= \max \left\{  \frac{r_1}{r_0}  , \dots,  \frac{r_{d-2}}{r_{d-3}} \right\}$. Assume that $n > \max\{m,200\} $ and $r^{\prime} \leq \min\{\frac{n}{6},  r_{d-2}\}$ hold. Moreover, assume that the sampling probability satisfies
\begin{eqnarray}\label{probsampboundfin}
p > \frac{1}{n^{d-2}} \max\left\{27 \ \log \left( \frac{n}{\epsilon} \right) + 9 \ \log \left( \frac{2M}{\epsilon} \right) + 18, 6r_{d-2}\right\} + \frac{1}{\sqrt[4]{n^{d-2}}}
\end{eqnarray}
Then, with probability at least $ (1- \epsilon) \left( 1-\exp(-\frac{\sqrt{n^{d-2}}}{2}) \right)^{n^2}$, $\mathcal{U}$ is finitely completable.
\end{lemma}

\begin{proof}
According to Lemma \ref{azumares}, assumption \eqref{probsampboundfin} results that each column of $\widetilde{\mathbf{\Omega}}_{(d-2)}$ includes at least $l$ nonzero entries, where $l$ satisfies \eqref{minlforset3th} with probability at least $\left( 1-\exp(-\frac{\sqrt{n^{d-2}}}{2}) \right)$. Therefore, with probability at least $\left( 1-\exp(-\frac{\sqrt{n^{d-2}}}{2}) \right)^{n^2}$, all $n^2$ columns of $\widetilde{\mathbf{\Omega}}_{(d-2)}$ satisfy \eqref{minlforset3th}. Hence, according to Theorem \ref{mainthsam}, with probability at least $ (1- \epsilon) \left( 1-\exp(-\frac{\sqrt{n^{d-2}}}{2}) \right)^{n^2}$, $\mathcal{U}$ is finitely completable.
\end{proof}


\section{Deterministic and Probabilistic Conditions for Unique Completability}\label{secsuni}

As we showed in \cite{ashraphijuo}, for matrix and tensor completion problems, finite completability does not necessarily imply unique completability. Theorem \ref{detconfinTT} and Theorem \ref{mainthsam} characterize the deterministic and probabilistic conditions on the sampling pattern $\Omega$ for finite completability, respectively. In this section, we add some additional mild restrictions on $\Omega$ and the number of samples to ensure unique completability. To this end, we obtain multiple sets of minimally algebraically dependent polynomials and show that the variables involved in these polynomials can be determined uniquely, and therefore entries of $\mathcal{U}$ can be determined uniquely. The following lemma is a re-statement of Lemma $9$ in \cite{ashraphijuo}.

\begin{lemma}\label{lemuniminindpol}
Suppose that Assumption $1$ holds. Let ${\breve{\Omega}}^{\prime} \in \mathbb{R}^{n_1 \times n_2 \times \dots \times n_{d-1} \times t}$ be an arbitrary subtensor of the constraint tensor ${\breve{\Omega}}$. Assume that polynomials in $\mathcal{P}({\breve{\Omega}}^{\prime})$ are minimally algebraically dependent. Then, all variables (unknown entries) of $\mathcal{U}^{(1)}$, $\mathcal{U}^{(2)}, \dots,$ and $\mathcal{U}^{(d-1)}$ that are involved in $\mathcal{P}({\breve{\Omega}}^{\prime})$ can be determined uniquely.
\end{lemma}

We explain the key point behind the proof of the following theorem. Condition (i) results in $  \sum_{i=1}^{d-1} r_{i-1}n_ir_i -\sum_{i=1}^{d-1} r_i^2 $ algebraically independent polynomials in terms of the entries of $\mathcal{U}^{(1)}$, $\mathcal{U}^{(2)}, \dots,$ and $\mathcal{U}^{(d-1)}$, i.e., results in finite completability. As a result, adding any single polynomial to these $  \sum_{i=1}^{d-1} r_{i-1}n_ir_i -\sum_{i=1}^{d-1} r_i^2 $ algebraically independent polynomials results in a set of algebraically dependent polynomials and according to Lemma \ref{lemuniminindpol} some of the entries of $\mathcal{U}^{(1)}$, $\mathcal{U}^{(2)}, \dots,$ and $\mathcal{U}^{(d-1)}$ can be determined uniquely. Then, condition (ii) results in more polynomials such that all entries of $\mathcal{U}^{(1)}$, $\mathcal{U}^{(2)}, \dots,$ and $\mathcal{U}^{(d-1)}$ can be determined uniquely.

\begin{theorem}\label{detunimain}
Suppose that Assumption $1$ holds. Also, assume that there exist disjoint subtensors ${\breve{\Omega}}^{\prime} \in \mathbb{R}^{n_1 \times n_2 \times \cdots \times n_{d-1} \times M}$ and ${\breve{\Omega}}^{{\prime}^i} \in \mathbb{R}^{n_1 \times n_2 \times \cdots \times n_{d-1} \times M_i}$ (for $1 \leq i \leq d-1$) of the constraint tensor such that the following conditions hold: 

(i)  \ $ M = \sum_{k=1}^{d-1} r_{k-1}n_kr_k -\sum_{k=1}^{d-1} r_k^2 $, and for any $t \in \{1,\dots,M\}$ and any subtensor ${\breve{\Omega}}^{\prime \prime} \in \mathbb{R}^{n_1 \times n_2 \times \cdots \times n_{d-1} \times t}$ of the tensor ${\breve{\Omega}}^{\prime}$, the following inequality holds
\begin{eqnarray}\label{ineqpunicon1}
\sum_{k=1}^{d-1} \left(r_{k-1}r_km_k({\breve{\Omega}}^{\prime \prime}) - r_k^2\right)^{+} \geq t.
\end{eqnarray}

(ii) for each $i \in \{1,\dots,d-1\}$ we have $M_i = n_i - \lfloor \frac{r_i}{r_{i-1}} \rfloor$, and for any $t_i \in \{1,\dots,M_i\}$ and any subtensor ${\breve{\Omega}}^{{\prime \prime}^i} \in \mathbb{R}^{n_1 \times n_2 \times \cdots \times n_{d-1} \times t_i}$ of the tensor ${\breve{\Omega}}^{{\prime}^i}$, the following inequality holds
\begin{eqnarray}\label{ineqpunicon2}
m_i({\breve{\Omega}}^{\prime \prime}) - \frac{r_i}{r_{i-1}} \geq t_i - \frac{r_i}{r_{i-1}}(t_i-M_i+1)^+.
\end{eqnarray}

Then, for almost every $\mathcal{U}$, there exists only a unique tensor that fits in the sampled tensor $\mathcal{U}$, and has TT rank $(r_1,r_2,\dots,r_{d-1})$.
\end{theorem}

\begin{proof}
According to the proof of Theorem \ref{detconfinTT}, $\mathcal{P}(\breve{\Omega}^{\prime})$ includes $M= \sum_{i=1}^{d-1} r_{i-1}n_ir_i -\sum_{i=1}^{d-1} r_i^2$ algebraically independent polynomials which results the finite completability of the sampled tensor $\mathcal{U}$ and let $\{ p_1, \dots , p_M \}$ denote these $M$ algebraically independent polynomials. Also, $M$ is the number of total variables among the polynomials, and therefore adding any polynomial $p_0$ to $\{ p_1, \dots , p_M \}$ results in a set of algebraically dependent polynomials. As a result, there exists a set of polynomials $\mathcal{P}(\breve{\Omega}^{\prime \prime})$ such that $\mathcal{P}(\breve{\Omega}^{\prime \prime}) \subset \{ p_1, \dots , p_M \}$ and also polynomials in $\mathcal{P}(\breve{\Omega}^{\prime \prime}) \cup p_0$ are minimally algebraically dependent polynomials. Hence, according to Lemma \ref{lemuniminindpol}, all the variables involved in the polynomials $\mathcal{P}(\breve{\Omega}^{\prime \prime}) \cup p_0$ can be determined uniquely. As a result, all variables involved in $p_0$ can be determined uniquely.

We can repeat the above procedure for any polynomial $p_0 \in \mathcal{P} ({\breve{\Omega}}^{{\prime}^i})$ to determine the involved variables uniquely with the help of $\{ p_1, \dots , p_M \}$, $i = 1,\dots,d-1$. Hence, we obtain $\sum_{k=1}^{d-1}r_{k-1}r_k $ polynomials but some of the entries of TT decomposition are elements of the $\mathbf{Q}_i$  matrices (in the statement of Lemma \ref{pattern}). In order to complete the proof, we need to show that condition (ii) with the above procedure using $\{ p_1, \dots , p_M \}$ results in obtaining all variables uniquely. In particular, we show that polynomials in $\mathcal{P} ({\breve{\Omega}}^{{\prime}^i})$ result in obtaining all variables of the $i$-th element of TT decomposition uniquely.

Note that since $t_i \leq M_i$, we have $(t_i-M_i+1)^+ = 1$ if $t_i-M_i=0$ and $(t_i-M_i+1)^+ = 0$ otherwise. Hence, if $t_i < M_i$ condition (ii) can be written as 
\begin{eqnarray}\label{ineqpunicon2i}
r_{i-1}r_im_i({\breve{\Omega}}^{\prime \prime}) - {r_i}^2 \geq r_{i-1}r_it_i,
\end{eqnarray}
which certifies the algebraically independence of the corresponding polynomials obtained by the mentioned procedure. Observe that we need $r_{i-1}r_in_i-r_i^2$ algebraically independent polynomials and in the case that $t_i=M_i$, condition (ii) results in $r_{i-1}r_in_i-r_i^2$ algebraically independent polynomials.
\end{proof}

Theorem \ref{detunimain} provides the deterministic condition on the sampling pattern $\Omega$ for unique completability. Using Theorem \ref{detunimain} we provide a bound on the number of samples to ensure unique completability with high probability. We first need to extended some of the lemmas in Section \ref{secprob} to obtain a condition on the number of samples to ensure condition (ii) in the statement of Theorem \ref{detunimain} holds with high probability. Note that Condition (i) is the same condition for finite completability and we already have the corresponding bound. 

In the rest of this section, for the sake of simplicity, as in Section \ref{secprob} we consider the sampled tensor $\mathcal{U} \in \mathbb{R}^{\overbrace {n \times \dots \times n}^{d}}$.

\begin{lemma}\label{extceilfloor}
Assume that $r^{\prime} \leq \frac{n}{6}$ and also each column of $\mathbf{\Omega}_{(1)}$ (first matricization of $\Omega$) includes at least $l$ nonzero entries, where 
\begin{eqnarray}\label{genminl1ceilfloor}
l > \max\left\{21 \ \log \left( \frac{n}{\epsilon} \right) + 3 \ \log \left( \frac{k}{\epsilon} \right) + 6, 2r^{\prime}\right\}. 
\end{eqnarray}
Let $\mathbf{\Omega}^{\prime}_{(1)}$ be an arbitrary set of $n -r^{\prime}+1$ columns of $\mathbf{\Omega}_{(1)}$. Then, with probability at least $1-\frac{\epsilon}{k}$, every {\bf proper} subset $\mathbf{\Omega}^{\prime \prime}_{(1)}$ of columns of $\mathbf{\Omega}^{\prime}_{(1)}$ satisfies 
\begin{eqnarray}\label{genproper1ceilfloor}
m_{1}({\Omega}^{\prime \prime}) - r^{\prime} \geq t,
\end{eqnarray}
where $t$ is the number of columns of $\mathbf{\Omega}^{\prime \prime}_{(1)}$ and $m_{1}({\Omega}^{\prime \prime})$ is the number of nonzero rows of $\mathbf{\Omega}^{\prime \prime}_{(1)}$.
\end{lemma}

\begin{proof}
Note that \eqref{genminl1ceilfloor} results the following
\begin{eqnarray}\label{genminl1ceilfloorreform}
l > \max\left\{9 \ \log \left( \frac{n}{\frac{\epsilon}{n}} \right) + 3 \ \log \left( \frac{k}{\frac{\epsilon}{n}} \right) + 6, 2r^{\prime}\right\}. 
\end{eqnarray}
Consider $n -r^{\prime}$ columns of $\mathbf{\Omega}^{\prime}_{(1)}$. According to Lemma \ref{genlem}, with probability at least $1-\frac{\epsilon}{nk}$, any subset of columns $\mathbf{\Omega}^{\prime \prime}_{(1)}$ of these $n -r^{\prime}$ particular columns of $\mathbf{\Omega}^{\prime}_{(1)}$ satisfies \eqref{genproper1ceilfloor}. Since there are $n$ possible subsets of columns of $\mathbf{\Omega}^{\prime}_{(1)}$ with $n -r^{\prime}$ columns, with probability at least $1-\frac{\epsilon}{k}$, every {\bf proper} subset $\mathbf{\Omega}^{\prime \prime}_{(1)}$ of columns of $\mathbf{\Omega}^{\prime}_{(1)}$ satisfies \eqref{genproper1ceilfloor}.
\end{proof}

\begin{lemma}\label{lemman2pceil}
Assume that $r^{\prime} \leq \frac{n}{6}$ and also let $j \in \{1,2,\dots,d-1\}$ be a fixed number. Consider an arbitrary set $\widetilde{\mathbf{\Omega}}^{\prime}_{(j)}$ of $n -r^{\prime}$ columns of $\widetilde{\mathbf{\Omega}}_{(j)}$ ($j$-th unfolding of $\Omega$). Assume that $n > \max \{400  , \sum_{k=1}^{d-1} r_{k-1}r_{k} \}$, and also each column of $\widetilde{\mathbf{\Omega}}_{(j)}$ includes at least $l$ nonzero entries, where 
\begin{eqnarray}\label{minlforset2pnceil}
l > \max\left\{63 \ \log \left( \frac{n}{\epsilon} \right) + 9 \ \log \left( \frac{2r}{\epsilon} \right) + 18, 6r^{\prime}\right\},
\end{eqnarray}
where $r \leq \sum_{k=1}^{d-1} r_{k-1}r_{k}$ (recall that $r_0=r_d=1$). Then, with probability at least $1-\frac{\epsilon}{2r}$, each column of $\widetilde{\mathbf{\Omega}}^{\prime}_{(j)}$ includes more than $\max\left\{21 \ \log \left( \frac{n}{\epsilon} \right) + 3 \ \log \left( \frac{2r}{\epsilon} \right) + 6, 2r^{\prime}\right\}$ observed entries of $\Omega$ with different values of the $i$-th coordinate, $1 \leq i \leq j$.
\end{lemma}

\begin{proof}
The proof is similar to the proof of Lemma \ref{lemman2p} and the only difference is in the calculations of $P(\zeta)$, where for this lemma $n > 400 $ is needed instead of $ n > 200$.
\end{proof}

\begin{lemma}\label{lemman2ceil}
Let $j \in \{1,2,\dots,d-1\}$ be a fixed number. Assume that $r_{i}^{\prime} \leq \frac{n}{6}$, where $r_i^{\prime}$ is rational and non-integer and also $i \in \{1,\dots,j\}$. Consider an arbitrary set $\widetilde{\mathbf{\Omega}}^{\prime}_{(j)}$ of $n - \lfloor r_i^{\prime} \rfloor$ columns of $\widetilde{\mathbf{\Omega}}_{(j)}$. Assume that $n > \max \{ 400 , \sum_{k=1}^{d-1} r_{k-1}r_{k} \}$, and also each column of $\widetilde{\mathbf{\Omega}}_{(j)}$ includes at least $l$ nonzero entries, where 
\begin{eqnarray}\label{minlforset2pneceil}
l > \max\left\{63 \ \log \left( \frac{n}{\epsilon} \right) + 9 \ \log \left( \frac{2r}{\epsilon} \right) + 18, 6 \lceil r_i^{\prime} \rceil \right\},
\end{eqnarray}
where $r \leq \sum_{k=1}^{d-1} r_{k-1}r_{k}$ (recall that $r_0=r_d=1$). Then, with probability  at least $1-\frac{\epsilon}{r}$, every {\bf proper} subset $\widetilde{\mathbf{\Omega}}^{\prime \prime}_{(j)}$ of columns of $\widetilde{\mathbf{\Omega}}^{\prime}_{(j)}$ satisfies 
\begin{eqnarray}\label{proper2ceil}
m_{i}({\Omega}^{\prime \prime}) - \lceil r_i^{\prime} \rceil \geq t,
\end{eqnarray}
where $t$ is the number of columns of $\widetilde{\mathbf{\Omega}}^{\prime \prime}_{(j)}$ and $\Omega^{\prime \prime}$ is the corresponding tensor such that $\widetilde{\mathbf{\Omega}}^{\prime \prime}_{(j)}$ is the $j$-th unfolding of $\Omega^{\prime \prime}$.
\end{lemma}

\begin{proof}
Each column of $\widetilde{\mathbf{\Omega}}_{(j)}$ includes $n^j$ entries and they can be represented by $(x_1,\dots,x_j)$ for $1 \leq x_k \leq n$ and $1 \leq k \leq j$, where $x_k$ denotes the $k$-th coordinate of the corresponding entry. According to Lemma \ref{lemman2pceil}, with probability  at least $1-\frac{\epsilon}{2r}$, each column of $\widetilde{\mathbf{\Omega}}_{(j)}$ includes more than $\max\left\{21 \ \log \left( \frac{n}{\epsilon} \right) + 3 \ \log \left( \frac{2r}{\epsilon} \right) + 6, 2 \lceil r_i^{\prime} \rceil \right\}$ observed entries with different values of the $i$-th coordinate. Therefore, as $\lceil r_i^{\prime} \rceil = \lfloor r_i^{\prime} \rfloor +1$ and according to Lemma \ref{extceilfloor}, with probability  at least $(1-\frac{\epsilon}{2r})^2$ which is more than $1-\frac{\epsilon}{r}$, every {\bf proper} subset $\widetilde{\mathbf{\Omega}}^{\prime \prime}_{(j)}$ of columns of $\widetilde{\mathbf{\Omega}}^{\prime}_{(j)}$ satisfies \eqref{proper2ceil}.
\end{proof}

\begin{theorem}\label{mainthsamuni}
Define $m= \sum_{k=1}^{d-2} r_{k-1}r_{k}$, $M = n \sum_{k=1}^{d-2} r_{k-1}r_k -\sum_{k=1}^{d-2} r_k^2$ and $r^{\prime}= \max \left\{  \frac{r_1}{r_0}  , \dots,  \frac{r_{d-2}}{r_{d-3}} \right\}$. Assume that $n > \max\{m+ d,400\} $ and $r^{\prime} \leq \min\{\frac{n}{6},  r_{d-2}\}$ hold. Moreover, assume that each column of $\widetilde{\mathbf{\Omega}}_{(d-2)}$ includes at least $l$ nonzero entries, where 
\begin{eqnarray}\label{minlforset3thuni}
l > \max\left\{63 \ \log \left( \frac{4n}{\epsilon} \right) + 9 \ \log \left( \frac{8M}{\epsilon} \right) + 18, 6r_{d-2}\right\}. 
\end{eqnarray}
Then, with probability at least $1-\epsilon$, for almost every $\mathcal{U} \in \mathbb{R}^{\overbrace {n \times \dots \times n}^{d}}$, there exists only one completion of the sampled tensor $\mathcal{U}$ with rank vector $(r_1,r_2,\dots,r_{d-1})$.
\end{theorem}

\begin{proof}
According to Theorem \ref{mainthsam}, with probability at least $1 - \frac{\epsilon}{4}$, condition (i) in the statement of Theorem \ref{detunimain} holds. Moreover, as $M > d$ and according to Lemma \ref{lemman2ceil}, with probability at least $1- \frac{\epsilon}{2d}$, condition (ii) holds for each $i$. Therefore, with probability at least $1-\epsilon$, conditions (i) and (ii) in the statement of Theorem \ref{detunimain} hold.
\end{proof}

\begin{remark}\label{remsamu}
A tensor $\mathcal{U}$ that satisfies the properties in the statement of Theorem \ref{mainthsamuni} requires 
\begin{eqnarray}\label{rremsamu}
n^2 \max\left\{63 \ \log \left( \frac{4n}{\epsilon} \right) + 9 \ \log \left( \frac{8M}{\epsilon} \right) + 18, 6r_{d-2}\right\}
\end{eqnarray}
samples to be uniquely completable with probability at least $1-\epsilon$ since $\widetilde{\mathbf{\Omega}}_{(d-2)}$ has $n^2$ columns, where $M = n \sum_{k=1}^{d-2} r_{k-1}r_k -\sum_{k=1}^{d-2} r_k^2$. Note that the number of samples given in Theorem $3$ of \cite{charact} results in both finite and unique completability, and therefore the number of samples required by the unfolding approach given  in Remark \ref{remsam0} is for both finite and unique completability.
\end{remark}

\begin{lemma}\label{probsamTTuniq}
Define $m= \sum_{k=1}^{d-2} r_{k-1}r_{k}$, $M = n \sum_{k=1}^{d-2} r_{k-1}r_k -\sum_{k=1}^{d-2} r_k^2$ and $r^{\prime}= \max \left\{  \frac{r_1}{r_0}  , \dots,  \frac{r_{d-2}}{r_{d-3}} \right\}$. Assume that $n > \max\{m+d,400\} $ and $r^{\prime} \leq \min\{\frac{n}{6},  r_{d-2}\}$ hold. Moreover, assume that the sampling probability satisfies
\begin{eqnarray}\label{probsampbounduniq}
p > \frac{1}{n^{d-2}} \max\left\{63 \ \log \left( \frac{4n}{\epsilon} \right) + 9 \ \log \left( \frac{8M}{\epsilon} \right) + 18, 6r_{d-2}\right\} + \frac{1}{\sqrt[4]{n^{d-2}}}
\end{eqnarray}
Then, with probability at least $ (1- \epsilon) \left( 1-\exp(-\frac{\sqrt{n^{d-2}}}{2}) \right)^{n^2}$, $\mathcal{U}$ is uniquely completable.
\end{lemma}

\begin{proof}
Using Theorem \ref{mainthsamuni}, the proof is similar to the proof of Lemma \ref{probsamTTfin}.
\end{proof}

\section{Numerical Comparisons}\label{secsimu}

In this section, we compute the total number of samples that is required for finiteness/uniqueness using an example to compare the unfolding approach and the TT approach. In this numerical example, we consider a $7$-way tensor $\mathcal{U} $ ($d=7$) such that each dimension size is $n=10^3$. We also consider the TT rank $(r_1,r_2, \dots ,r_6)=(r,2r,3r,3r,2r,r)$ and $(r_1,r_2, \dots ,r_6)=(r,r^2,r^3,r^3,r^2,r)$ in Figures \ref{fig1} and \ref{fig2}, respectively. Figures \ref{fig1} and \ref{fig2} plot the bounds given in Remark \ref{remsam0} (unfolding approach for either finite or unique completability), Remark \ref{remsam} (TT approach for finite completability), and Remark \ref{remsamu} (TT approach for unique completability) for the corresponding rank vector, where  $\epsilon = 0.001$. We change the value of $r$ from $1$ to $80$ which is denoted by ``rank" in Figure \ref{fig1} and from $1$ to $20$ in Figure \ref{fig2}. It is seen that the number of samples required by the proposed TT approach is substantially lower than that is required by the unfolding approach.  

\begin{figure}
	\centering
		{\includegraphics[width=11cm]{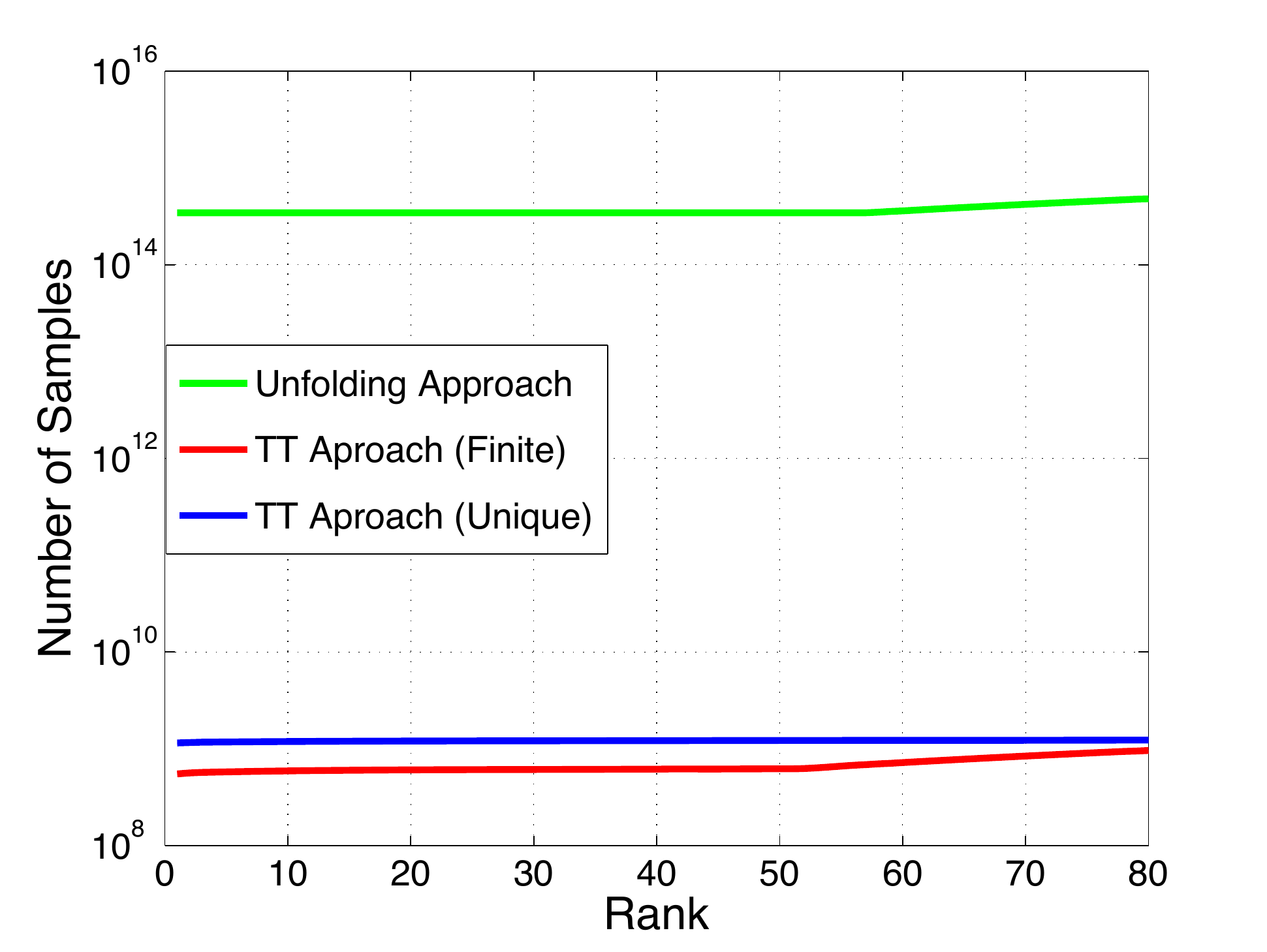}}
	\caption{ Lower bounds on the number of samples for a $7$-way tensor with rank vector $(r,2r,3r,3r,2r,r)$.}
	\label{fig1}\vspace{-4mm}
\end{figure}

\begin{figure}
	\centering
		{\includegraphics[width=11cm]{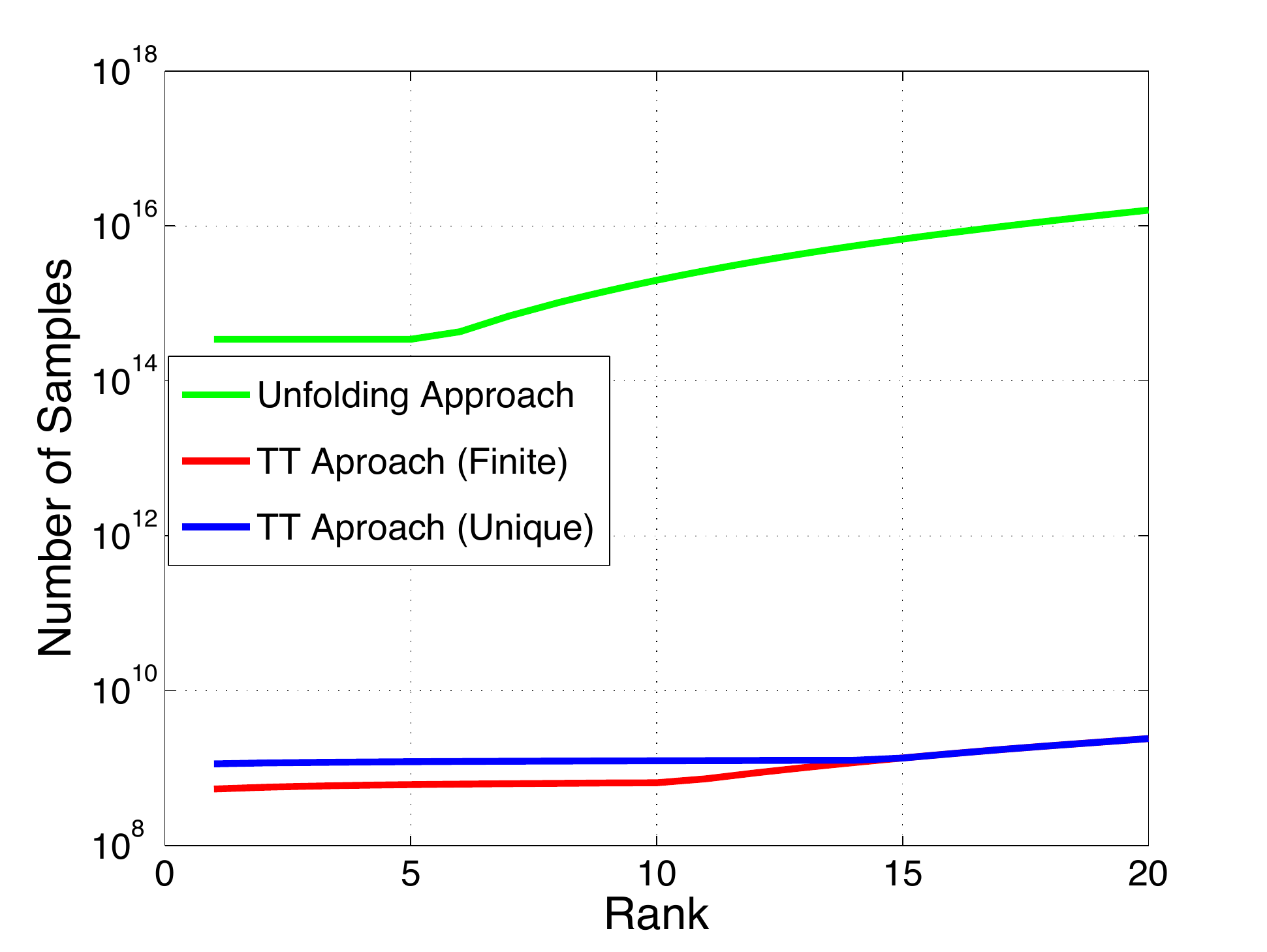}}
	\caption{ Lower bounds on the number of samples for a $7$-way tensor with rank vector $(r,r^2,r^3,r^3,r^2,r)$.}
	\label{fig2}\vspace{-4mm}
\end{figure}

\section{Conclusions}\label{seccon}

This paper characterizes fundamental conditions on the sampling pattern for finite completability of a low TT rank and partially sampled tensor through a new algebraic geometry analysis on the TT manifold. We defined a polynomial based on each sampled entry and exploited the structure of the TT decomposition to study the algebraic independence of these polynomials  based on the locations of the  samples. We also developed a geometric pattern on the TT decomposition, which can be treated as an equivalence class that partitions all TT decompositions of one particular tensor to different classes. This equivalence class is helpful to study the algebraic independence of the defined polynomials. Using the developed tools on the TT manifold, we characterized the maximum number of algebraically independent polynomials among all the defined polynomials in terms of a simple geometric structure of the sampling pattern. Our analysis results in the following fundamental conditions for low-TT-rank tensor completion: (i) The necessary and sufficient deterministic  conditions on the sampling pattern, under which there are only finite completions given the TT rank, (ii) Deterministic  sufficient conditions on the sampling pattern, under which there exists exactly one completion given the TT rank, (iii) Lower bounds on the number of samples that leads to finite/unique completability with high probability.

\bibliographystyle{IEEETran}
\bibliography{bib}

\begin{thebibliography}{10}
\providecommand{\url}[1]{#1}
\csname url@samestyle\endcsname
\providecommand{\newblock}{\relax}
\providecommand{\bibinfo}[2]{#2}
\providecommand{\BIBentrySTDinterwordspacing}{\spaceskip=0pt\relax}
\providecommand{\BIBentryALTinterwordstretchfactor}{4}
\providecommand{\BIBentryALTinterwordspacing}{\spaceskip=\fontdimen2\font plus
\BIBentryALTinterwordstretchfactor\fontdimen3\font minus
  \fontdimen4\font\relax}
\providecommand{\BIBforeignlanguage}[2]{{%
\expandafter\ifx\csname l@#1\endcsname\relax
\typeout{** WARNING: IEEEtran.bst: No hyphenation pattern has been}%
\typeout{** loaded for the language `#1'. Using the pattern for}%
\typeout{** the default language instead.}%
\else
\language=\csname l@#1\endcsname
\fi
#2}}
\providecommand{\BIBdecl}{\relax}
\BIBdecl

\bibitem{candes}
E.~J. Cand{\`e}s and B.~Recht, ``Exact matrix completion via convex
  optimization,'' \emph{Foundations of Computational Mathematics}, vol.~9,
  no.~6, pp. 717--772, 2009.

\bibitem{candes2}
E.~J. Cand{\`e}s and T.~Tao, ``The power of convex relaxation: Near-optimal
  matrix completion,'' \emph{IEEE Transactions on Information Theory}, vol.~56,
  no.~5, pp. 2053--2080, 2010.

\bibitem{cai}
J.~F. Cai, E.~J. Cand{\`e}s, and Z.~Shen, ``A singular value thresholding
  algorithm for matrix completion,'' \emph{SIAM Journal on Optimization},
  vol.~20, no.~4, pp. 1956--1982, 2010.

\bibitem{ashraphijuo2016c}
M.~Ashraphijuo, R.~Madani, and J.~Lavaei, ``Characterization of
  rank-constrained feasibility problems via a finite number of convex
  programs,'' in \emph{IEEE 55th Conference on Decision and Control (CDC)},
  2016, pp. 6544--6550.

\bibitem{phase}
E.~J. Cand{\`e}s, Y.~C. Eldar, T.~Strohmer, and V.~Voroninski, ``Phase
  retrieval via matrix completion,'' \emph{SIAM Journal on Imaging Sciences},
  vol.~6, no.~1, pp. 199--225, 2013.

\bibitem{jain2013low}
P.~Jain, P.~Netrapalli, and S.~Sanghavi, ``Low-rank matrix completion using
  alternating minimization,'' in \emph{Annual Symposium on the Theory of
  Computing}, 2013, pp. 665--674.

\bibitem{ge2016matrix}
R.~Ge, J.~D. Lee, and T.~Ma, ``Matrix completion has no spurious local
  minimum,'' \emph{arXiv preprint:1605.07272}, 2016.

\bibitem{gandy}
S.~Gandy, B.~Recht, and I.~Yamada, ``Tensor completion and low-n-rank tensor
  recovery via convex optimization,'' \emph{Inverse Problems}, vol.~27, no.~2,
  pp. 1--19, 2011.

\bibitem{tomioka}
R.~Tomioka, K.~Hayashi, and H.~Kashima, ``Estimation of low-rank tensors via
  convex optimization,'' \emph{arXiv preprint:1010.0789}, 2010.

\bibitem{nuctensor}
M.~Signoretto, Q.~T. Dinh, L.~De~Lathauwer, and J.~A. Suykens, ``Learning with
  tensors: a framework based on convex optimization and spectral
  regularization,'' \emph{Machine Learning}, vol.~94, no.~3, pp. 303--351,
  2014.

\bibitem{romera}
B.~Romera-Paredes and M.~Pontil, ``A new convex relaxation for tensor
  completion,'' in \emph{Advances in Neural Information Processing Systems},
  2013, pp. 2967--2975.

\bibitem{low}
D.~Kressner, M.~Steinlechner, and B.~Vandereycken, ``Low-rank tensor completion
  by {R}iemannian optimization,'' \emph{BIT Numerical Mathematics}, vol.~54,
  no.~2, pp. 447--468, 2014.

\bibitem{low2}
A.~Krishnamurthy and A.~Singh, ``Low-rank matrix and tensor completion via
  adaptive sampling,'' in \emph{Advances in Neural Information Processing
  Systems}, 2013, pp. 836--844.

\bibitem{goldfarb}
D.~Goldfarb and Z.~Qin, ``Robust low-rank tensor recovery: Models and
  algorithms,'' \emph{SIAM Journal on Matrix Analysis and Applications},
  vol.~35, no.~1, pp. 225--253, 2014.

\bibitem{7347424}
X.~Y. Liu, S.~Aeron, V.~Aggarwal, X.~Wang, and M.~Y. Wu, ``Adaptive sampling of
  {RF} fingerprints for fine-grained indoor localization,'' \emph{IEEE
  Transactions on Mobile Computing}, vol.~15, no.~10, pp. 2411--2423, 2016.

\bibitem{wang2016tensor}
W.~Wang, V.~Aggarwal, and S.~Aeron, ``Tensor completion by alternating
  minimization under the tensor train {(TT)} model,'' \emph{arXiv
  preprint:1609.05587}, 2016.

\bibitem{liulow2}
X.-Y. Liu, S.~Aeron, V.~Aggarwal, and X.~Wang, ``Low-tubal-rank tensor
  completion using alternating minimization,'' \emph{arXiv
  preprint:1610.01690}, 2016.

\bibitem{lim}
L.-H. Lim and P.~Comon, ``Multiarray signal processing: Tensor decomposition
  meets compressed sensing,'' \emph{Comptes Rendus Mecanique}, vol. 338, no.~6,
  pp. 311--320, 2010.

\bibitem{sid}
N.~D. Sidiropoulos and A.~Kyrillidis, ``Multi-way compressed sensing for sparse
  low-rank tensors,'' \emph{IEEE Signal Processing Letters}, vol.~19, no.~11,
  pp. 757--760, 2012.

\bibitem{visual}
J.~Liu, P.~Musialski, P.~Wonka, and J.~Ye, ``Tensor completion for estimating
  missing values in visual data,'' \emph{IEEE Transactions on Pattern Analysis
  and Machine Intelligence}, vol.~35, no.~1, pp. 208--220, 2013.

\bibitem{kreimer}
N.~Kreimer, A.~Stanton, and M.~D. Sacchi, ``Tensor completion based on nuclear
  norm minimization for {5D} seismic data reconstruction,'' \emph{Geophysics},
  vol.~78, no.~6, pp. V273--V284, 2013.

\bibitem{ely20135d}
G.~Ely, S.~Aeron, N.~Hao, M.~E. Kilmer \emph{et~al.}, ``5d and 4d pre-stack
  seismic data completion using tensor nuclear norm (tnn),'' in \emph{Society
  of Exploration Geophysicists}, 2013.

\bibitem{liu2016tensor}
X.-Y. Liu, S.~Aeron, V.~Aggarwal, X.~Wang, and M.-Y. Wu, ``Tensor completion
  via adaptive sampling of tensor fibers: Application to efficient indoor {RF}
  fingerprinting,'' in \emph{IEEE International Conference on Acoustics, Speech
  and Signal Processing (ICASSP)}, 2016, pp. 2529--2533.

\bibitem{vaneetcnsm}
V.~Aggarwal, A.~A. Mahimkar, H.~Ma, Z.~Zhang, S.~Aeron, and W.~Willinger,
  ``Inferring smartphone service quality using tensor methods,'' in
  \emph{International Conference on Network and Service Management}, 2016.

\bibitem{charact}
D.~Pimentel-Alarc{\'o}n, N.~Boston, and R.~Nowak, ``A characterization of
  deterministic sampling patterns for low-rank matrix completion,'' \emph{IEEE
  Journal of Selected Topics in Signal Processing}, vol.~10, no.~4, pp.
  623--636, 2016.

\bibitem{ashraphijuo2}
M.~Ashraphijuo, X.~Wang, and V.~Aggarwal, ``Deterministic and probabilistic
  conditions for finite completability of low-rank multi-view data,''
  \emph{arXiv preprint:1701.00737}, 2017.

\bibitem{ashraphijuo}
M.~Ashraphijuo, V.~Aggarwal, and X.~Wang, ``Deterministic and probabilistic
  conditions for finite completability of low rank tensor,'' \emph{arXiv
  preprint:1612.01597}, 2016.

\bibitem{Tuck}
T.~G. Kolda, ``Orthogonal tensor decompositions,'' \emph{SIAM Journal on Matrix
  Analysis and Applications}, vol.~23, no.~1, pp. 243--255, 2001.

\bibitem{SVD}
L.~Grasedyck, ``Hierarchical singular value decomposition of tensors,''
  \emph{SIAM Journal on Matrix Analysis and Applications}, vol.~31, no.~4, pp.
  2029--2054, 2010.

\bibitem{Tuckermanifold}
D.~Kressner, M.~Steinlechner, and B.~Vandereycken, ``Low-rank tensor completion
  by riemannian optimization,'' \emph{BIT Numerical Mathematics}, vol.~54,
  no.~2, pp. 447--468, 2014.

\bibitem{ten}
J.~M. ten Berge and N.~D. Sidiropoulos, ``On uniqueness in candecomp/parafac,''
  \emph{Psychometrika}, vol.~67, no.~3, pp. 399--409, 2002.

\bibitem{kruskal}
A.~Stegeman and N.~D. Sidiropoulos, ``On {K}ruskal’s uniqueness condition for
  the {C}andecomp/{P}arafac decomposition,'' \emph{Linear Algebra and its
  Applications}, vol. 420, no.~2, pp. 540--552, 2007.

\bibitem{kilmer2013third}
M.~E. Kilmer, K.~Braman, N.~Hao, and R.~C. Hoover, ``Third-order tensors as
  operators on matrices: A theoretical and computational framework with
  applications in imaging,'' \emph{SIAM Journal on Matrix Analysis and
  Applications}, vol.~34, no.~1, pp. 148--172, 2013.

\bibitem{Eck}
J.~D. Carroll and J.-J. Chang, ``Analysis of individual differences in
  multidimensional scaling via an n-way generalization of
  “{E}ckart-{Y}oung” decomposition,'' \emph{Psychometrika}, vol.~35, no.~3,
  pp. 283--319, 1970.

\bibitem{de2}
L.~De~Lathauwer, ``A survey of tensor methods,'' in \emph{IEEE International
  Symposium on Circuits and Systems}, 2009, pp. 2773--2776.

\bibitem{papa}
E.~E. Papalexakis, C.~Faloutsos, and N.~D. Sidiropoulos, ``Parcube: Sparse
  parallelizable tensor decompositions,'' in \emph{Joint European Conference on
  Machine Learning and Knowledge Discovery in Databases}, 2012, pp. 521--536.

\bibitem{beck2n}
M.~H. Beck, A.~J{\"a}ckle, G.~Worth, and H.-D. Meyer, ``The multiconfiguration
  time-dependent hartree ({MCTDH}) method: a highly efficient algorithm for
  propagating wavepackets,'' \emph{Physics Reports}, vol. 324, no.~1, pp.
  1--105, 2000.

\bibitem{scholy}
U.~Schollw{\"o}ck, ``The density-matrix renormalization group,'' \emph{in
  Journal of Modern Physics}, vol.~77, no.~1, p. 259, 2005.

\bibitem{oseledets}
I.~V. Oseledets, ``Tensor-train decomposition,'' \emph{SIAM Journal on
  Scientific Computing}, vol.~33, no.~5, pp. 2295--2317, 2011.

\bibitem{ose009king}
I.~V. Oseledets and E.~E. Tyrtyshnikov, ``Breaking the curse of dimensionality,
  or how to use {SVD} in many dimensions,'' \emph{SIAM Journal on Scientific
  Computing}, vol.~31, no.~5, pp. 3744--3759, 2009.

\bibitem{oselesor}
------, ``Tensor tree decomposition does not need a tree,'' \emph{Linear
  Algebra Applications}, vol.~8, 2009.

\bibitem{TT}
S.~Holtz, T.~Rohwedder, and R.~Schneider, ``On manifolds of tensors of fixed
  {TT}-rank,'' \emph{Numerische Mathematik}, vol. 120, no.~4, pp. 701--731,
  2012.

\bibitem{Bernstein}
B.~Sturmfels, \emph{Solving {S}ystems of {P}olynomial {E}quations}.\hskip 1em
  plus 0.5em minus 0.4em\relax American Mathematical Society, 2002, no.~97.

\end{thebibliography}

\end{document}